\documentclass{article}

\usepackage{microtype}
\usepackage{graphicx}
\usepackage{subcaption}
\usepackage{booktabs} 
\usepackage{adjustbox}
\usepackage{colortbl} 
\usepackage{placeins}

\newcommand{\TableFont}{\scriptsize}
\newcommand{\TableColSep}{3pt}

\newenvironment{icmltable}[1][\columnwidth]
{%
  \begingroup
  \TableFont
  \setlength{\tabcolsep}{\TableColSep}%
  \renewcommand{\arraystretch}{\TableRowStretch}%
  \centering
  \begin{adjustbox}{max width=#1}%
}
{%
  \end{adjustbox}%
  \endgroup
}

\usepackage[colorlinks=true, citecolor=green, linkcolor=red, urlcolor=blue]{hyperref}

\usepackage[accepted]{icml2026}

\usepackage{caption}
\usepackage{amsmath}
\usepackage{amssymb}
\usepackage{mathtools}
\usepackage{amsthm}
\usepackage[capitalize,noabbrev]{cleveref}

\usepackage[utf8]{inputenc}
\usepackage[T1]{fontenc}
\usepackage{multirow}
\usepackage{float}
\usepackage{wrapfig}
\usepackage{paralist}
\usepackage[table,xcdraw]{xcolor}
\usepackage{natbib}
\usepackage{pifont}

\usepackage{etoolbox}
\makeatletter
\csundef{algorithm}
\csundef{endalgorithm}
\csundef{algorithm*}
\csundef{endalgorithm*}
\csundef{listofalgorithms}
\makeatother

\usepackage[ruled,vlined]{algorithm2e}
\SetAlgoInsideSkip{0.4em}          
\SetAlgoSkip{0.6em}
\SetAlCapSkip{0.3em}

\newcommand{\WrapFigInTextSep}{6pt}    
\newcommand{\WrapFigColumnSep}{8pt}    
\newcommand{\WrapFigCaptionSkip}{2pt}  

{%
  \setlength{\intextsep}{\WrapFigInTextSep}%
  \setlength{\columnsep}{\WrapFigColumnSep}%
  \begin{wrapfigure}{#1}{#2}
    \captionsetup{font=small,skip=\WrapFigCaptionSkip}
}
{%
  \end{wrapfigure}
}

\usepackage[most]{tcolorbox}

\newtcolorbox{qbox}{
  colback=gray!10, colframe=black, fonttitle=\bfseries, title=Question,
  boxrule=0.5pt, arc=4pt, left=4pt, right=4pt, top=2pt, bottom=2pt
}
\newtcolorbox{abox}{
  colback=green!5, colframe=green!50!black, fonttitle=\bfseries, title=Answer,
  boxrule=0.5pt, arc=4pt, left=4pt, right=4pt, top=2pt, bottom=2pt
}
\newtcolorbox{graybox}{
  colback=gray!10, colframe=black!20, boxrule=0pt, arc=3pt,
  left=4pt, right=4pt, top=2pt, bottom=2pt, enhanced,
  sharp corners=downhill, before skip=6pt, after skip=6pt
}
\newtcolorbox{redbox}{
  colback=red!5, colframe=red!50!black, boxrule=0pt, arc=3pt,
  left=4pt, right=4pt, top=2pt, bottom=2pt, enhanced,
  sharp corners=downhill, before skip=6pt, after skip=6pt
}
\newtcolorbox{greenbox}{
  colback=green!5, colframe=green!40!black, boxrule=0pt, arc=3pt,
  left=4pt, right=4pt, top=2pt, bottom=2pt, enhanced,
  sharp corners=downhill, before skip=6pt, after skip=6pt
}

\def\beq#1\beq{\begin{equation}\begin{aligned}#1\end{aligned}\end{equation}}

\newcommand{\bea}{\begin{eqnarray}}
\newcommand{\eea}{\end{eqnarray}}

\theoremstyle{plain}
\newtheorem{theorem}{Theorem}[section]

\newtheorem{lemma}[theorem]{Lemma}

\theoremstyle{definition}

\theoremstyle{remark}

\definecolor{table-blue}{RGB}{173, 216, 230}

\icmltitlerunning{Adaptive Efficient Rollout Optimization for Group-based RL}

\begin{document}









\twocolumn[
\icmltitle{Train Less, Learn More: Adaptive Efficient\\
Rollout Optimization for Group-Based Reinforcement Learning}

\icmlsetsymbol{intern}{*}

\begin{icmlauthorlist}
\icmlauthor{Zhi Zhang}{ucla,intern}
\icmlauthor{Zhen Han}{aws}
\icmlauthor{Costas Mavromatis}{aws}
\icmlauthor{Qi Zhu}{aws}
\icmlauthor{Yunyi Zhang}{aws}
\icmlauthor{Sheng Guan}{aws}
\icmlauthor{Dingmin Wang}{aws}
\icmlauthor{Xiong Zhou}{aws}
\icmlauthor{Shuai Wang}{aws}
\icmlauthor{Soji Adeshina}{aws}
\icmlauthor{Vassilis Ioannidis}{aws}
\icmlauthor{Huzefa Rangwala}{aws}
\end{icmlauthorlist}

\icmlaffiliation{ucla}{University of California, Los Angeles}
\icmlaffiliation{aws}{Amazon Web Services}

\vspace{-3pt}
{\center
\textsuperscript{1}University of California, Los Angeles\quad
\textsuperscript{2}Amazon Web Services\par
}
\vspace{-10pt}

\icmlcorrespondingauthor{Zhen Han}{zhenhz@amazon.com}
\icmlcorrespondingauthor{Costas Mavromatis}{mavrok@amazon.com}

\icmlkeywords{Reinforcement Learning, Group-based RL, Efficiency}

\vskip 0.3in
]


\printAffiliationsAndNotice{%
\textsuperscript{*}Work done during an internship at Amazon Web Services (AWS); \texttt{zzh237@ucla.edu}.\\[-1pt]
Author contacts:\\[-2pt]
\begingroup
\footnotesize\ttfamily
\hyphenchar\font=`\-
\{zhenhz,\allowbreak mavrok,\allowbreak qzhuamzn,\allowbreak zhyunyi,\allowbreak shguan,\allowbreak wdimmy\}\allowbreak@\allowbreak amazon.com\\
\{xiongzho,\allowbreak wshui,\allowbreak adesojia,\allowbreak ivasilei,\allowbreak rhuzefa\}\allowbreak@\allowbreak amazon.com
\endgroup
\\[-1pt] 
}

\begin{abstract}
Reinforcement learning (RL) plays a central role in Large Language Model (LLM) post-training, among existing approaches, Group Relative Policy Optimization (GRPO) has become one of the most widely adopted methods, especially in RL with verifiable rewards (RLVR)-style fine-tuning.
In GRPO, each query prompts the LLM to generate a group of rollouts with a fixed group size $N$.  
When all rollouts in a group share the same outcome---either all correct or all incorrect---the group-normalized advantages become zero, yielding no gradient signal and wasting fine-tuning compute.
To mitigate this inefficiency, we introduce \textbf{Adaptive Efficient Rollout Optimization} (AERO). AERO addresses this by using an adaptive rollout strategy, 
applying selective rejection to  strategically prune rollouts, and maintaining a  Bayesian posterior to prevent zero-advantage dead zones. 
Across three model configurations (Qwen2.5-Math-1.5B, Qwen2.5-7B, and Qwen2.5-7B-Instruct), AERO significantly improves compute efficiency without sacrificing performance. 
Under the same total rollout budget, AERO reduces total training compute by $\sim$\textbf{48\%}  
while shortening wall-clock time per step by $\sim$\textbf{45\%} on average. 
Despite the substantial reduction in compute, AERO matches or improves
Pass@8 and Avg@8 over GRPO, demonstrating that it is a practical, scalable, and compute-efficient strategy for RL-based LLM alignment.
\end{abstract}

\newcommand{\costas}[1]{{\color{purple} #1}}
\newcommand{\zhen}[1]{{\color{pink} #1}}
\newcommand{\zach}[1]{{\color{orange} #1}}

\vspace{-0pt}
\section{Introduction}

RLVR has become a central paradigm for LLMs with factual and reasoning accuracy ~\citep{shao2024deepseekmath, liu2025understanding, yu2025dapo}.
A typical RLVR pipeline consists of two key stages: (1) the \emph{rollout phase}, where multiple candidate responses are sampled from the current policy, and (2) the \emph{training phase}, where verifiable rewards (e.g., correctness of answers) are used to update the model via policy gradients. 

\newcommand{\ZeroFigH}{3.0cm} 
\begin{figure}[t]
  \centering

  \begin{subfigure}[t]{0.44\linewidth}
    \centering
    \includegraphics[height=\ZeroFigH,keepaspectratio, width=\linewidth]{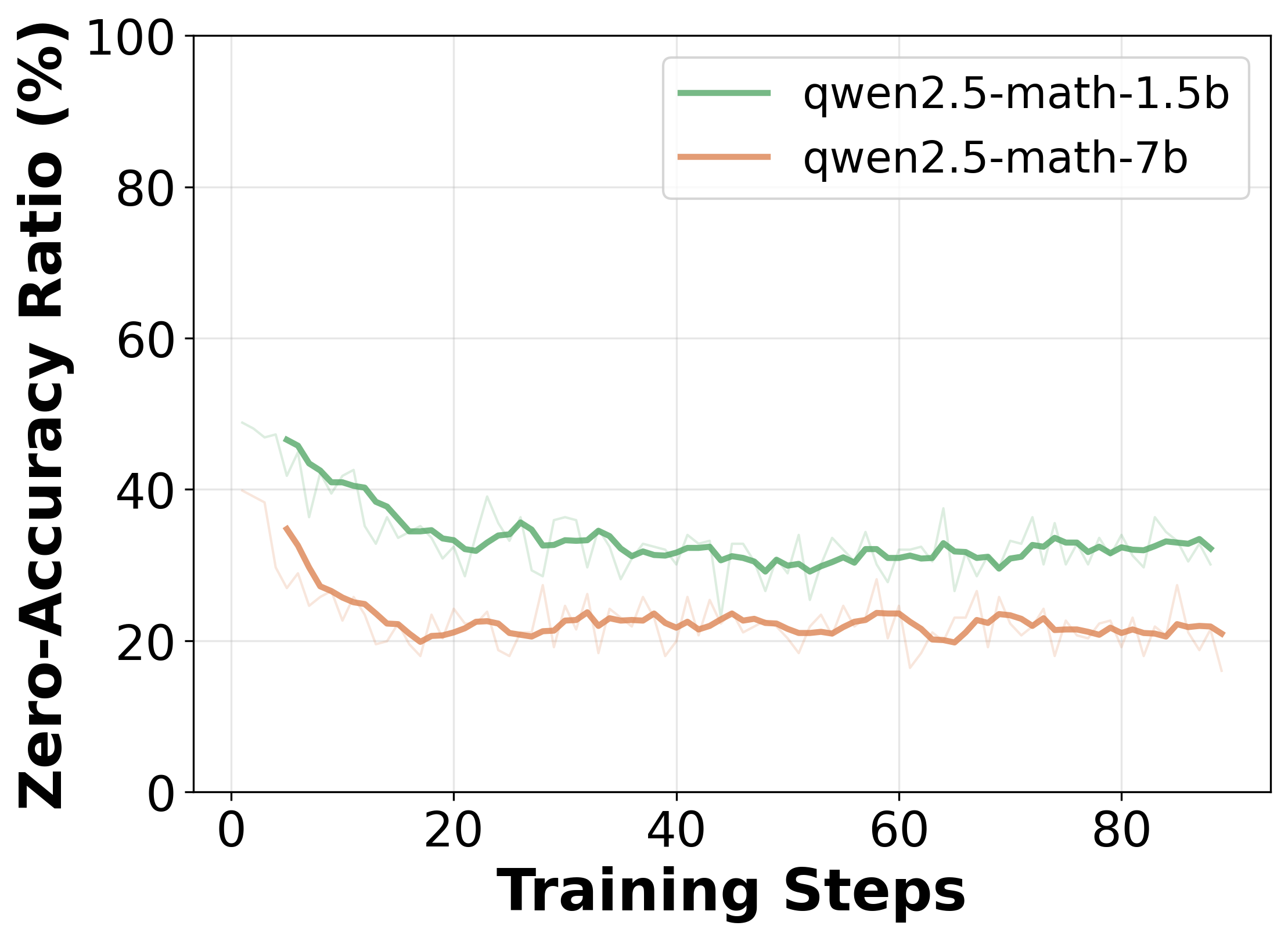}
    \caption{}
    \label{fig:zero_accuracy_problems}
  \end{subfigure}\hfill
  \begin{subfigure}[t]{0.48\linewidth}
    \centering
    \includegraphics[height=\ZeroFigH,keepaspectratio, width=\linewidth,,trim=2 6 6 5,clip]{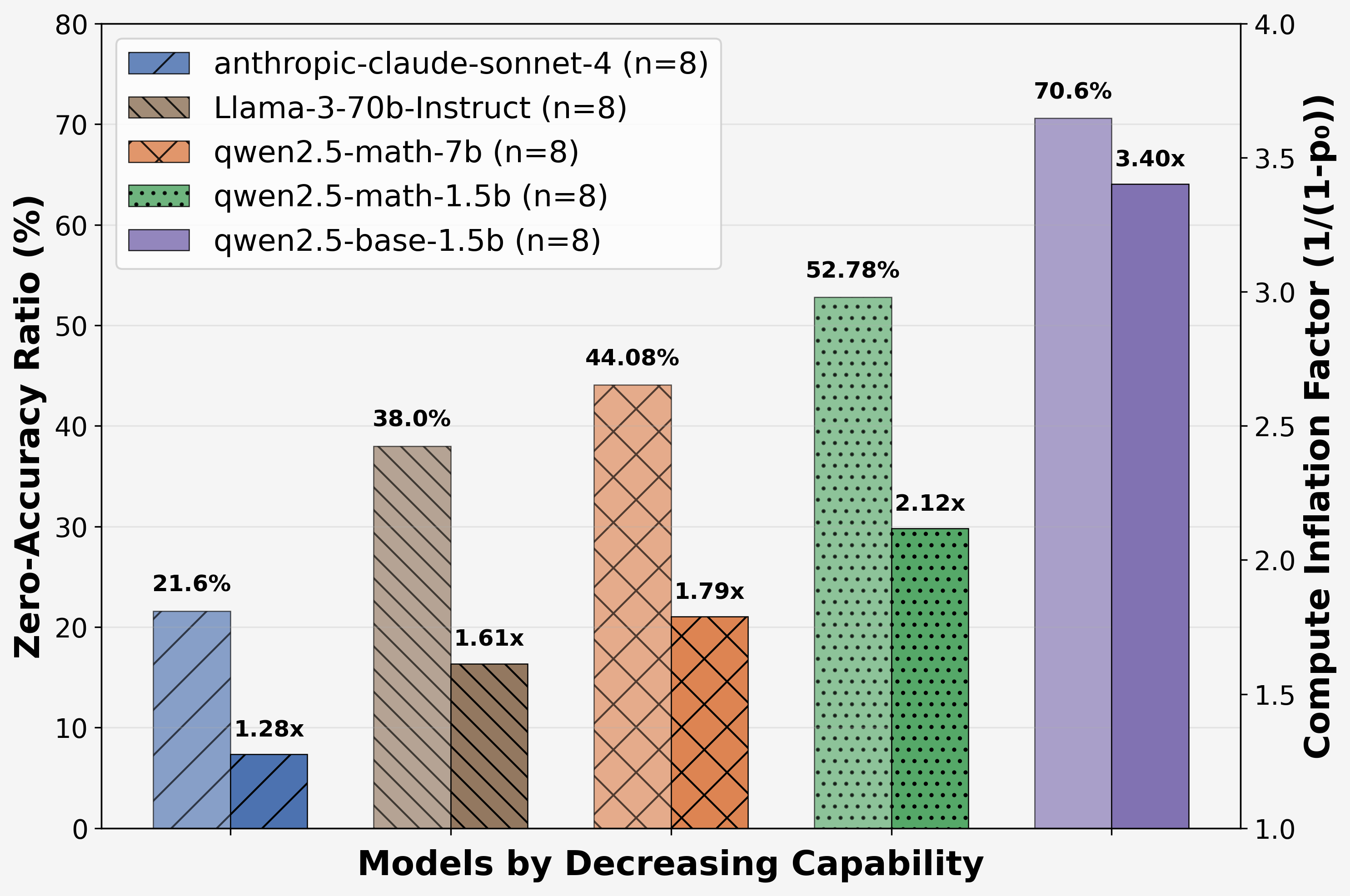}
    \caption{}
    \label{fig:zero_problem_ratio}
  \end{subfigure}

  \vspace{-0.4em}
  \caption{Zero-accuracy problems in GRPO-style training and inference:
  (\subref{fig:zero_accuracy_problems}) zero-accuracy problem ratio during GRPO training;
  (\subref{fig:zero_problem_ratio}) Proportion of zero-accuracy problems during rollout inference across model sizes, averaged over five benchmarks ($n=8$).
Light bars show the zero-accuracy ratio $p_0$ (\%, left y-axis). Dark bars show the compute inflation factor (right y-axis), defined as $1/(1-p_0)$, i.e., the expected oversampling multiplier needed to obtain the same number of non-zero-accuracy queries.
Text annotations (e.g., $1.28\times$) report the inflation factor.}
  \label{fig:zero_accuracy_two_panel}
  \vspace{-0.6em}
\end{figure}
A central challenge in RLVR is the high computational cost. Methods include GRPO, Dr.GRPO \citep{liu2025understanding}, RLOO \citep{ahmadian2024back} generate a large number of rollouts for a training query, which requires the additional compute. 
One major source of inefficiency arises when many queries contain no correct rollouts, producing \emph{zero advantages} and vanishing gradients, which waste computation and slow convergence. 
Figure~\ref{fig:zero_accuracy_problems} shows that in the RL post-training runs for 1.5B and 7B, more than 35\% of queries fall into this zero-advantage regime --- consuming substantial compute without contributing to learning. Figure~\ref{fig:zero_problem_ratio} shows that this also happens on large models such as Claude-Sonnet-4 and Llama-3-70b.    
When a fraction $p_0$ of queries yield zero accuracy, maintaining an effective training batch size requires sampling $1/(1-p_0)\times$ more queries on average, directly inflating rollout overhead. This highlights the inefficiency of static rollout in GRPO.


Recent studies \citep{yu2025dapo, lin2025cppo, zhang2025speed} have explored multiple directions to improve training efficiency, including more effective utilization of rollout budgets and selective use of samples during training. To mitigate the zero-advantage issue, DAPO~\citep{yu2025dapo} filters out zero-advantage problems entirely, to prevent using the zero-advantage queries. While CPPO filters
out rollouts with low absolute advantage~\citep{lin2025cppo}. However, filtering out or discarding zero- or one-accuracy queries still wastes the rollout budget already consumed for generation.   Lacking a priori measure of training problem
informativeness, these methods require the pre-sampling of an extensive batch of rollouts before each update iteration, potentially negating efficiency gains from the significant sampling overhead. Thus, this raises our first research question: \emph{How can we design an adaptive rollout strategy that improves compute efficiency without increasing the sampling overhead?}

Beyond inefficiencies from zero-advantage queries, once a query contains both correct and incorrect rollouts (e.g., 1 correct and 7 incorrect out of $N=8$), it may be unnecessary to retain all negatives for model updates. 
Several studies have explored heuristic forms of such selective downsampling to enhance performance \citep{shang2025rstar2agent, shrivastava2025sample}.
This naturally raises the question: \emph{Can subsampling be guided by a theoretically grounded, optimization-based principle rather than heuristics?}

Building upon the above unsolved problems, we introduce our method, Adaptive Efficient Rollout Optimization (AERO),  with three contributions.
(1) We replace fixed-$N$ with 
adaptive rollout allocation that increases budget on hard (all-fail) queries and cuts it on easy ones, avoiding the fixed-$N$ induced gradient dead zone by reducing the zero accuracy queries. We make sure that under the same probability, all the queries are able to obtain at least one correct rollout compared to the fixed-$N$ strategy, while our adaptive strategy results in fewer zero accuracy queries.  
(2) We maintain a Beta--Binomial posterior over each query’s success rate; the posterior avoid zero-advantage cases and improves sample efficiency instead of discarding them. Note, this approach is also applied to perfect queries to trigger non-zero advantages. 
(3)  For queries yield a mix of correct/incorrect rollouts with a low success rate, 
we keep all positives and \emph{down-sample} negatives to a target $1\!:\!k (k \geq 1)$ ratio via rejection sampling. 
We not only reduce compute by rejection but also theoretically show that an appropriate down-sampling ratio ($k=1$) enlarges the expected policy gradient norm thus to further prevent gradient dead zone. 


\textbf{Overall results.} 
AERO is an effective approach for improving GRPO’s computational efficiency through adaptive rollout budget reallocation and theoretically supported subsampling. 
By making rollout allocation adaptive, AERO significantly improves compute efficiency without sacrificing performance.  
Under the same total rollout budget, AERO reduces total training compute by \textbf{48\%} (from 3,714 to 1,917 PFLOPs) for 1.5B and by \textbf{47\%} (from 17,334 to 9,181 PFLOPs) for 7B,  and shortens per-step wall-clock time by \textbf{43\%} for 1.5B and \textbf{49\%} for 7B, while slightly improving final benchmark accuracy.  

\section{Related Works}
RLVR has become a cornerstone for enhancing the reasoning capabilities of LLMs, with GRPO emerging as a prominent and resource-efficient alternative to traditional Proximal Policy Optimization (PPO) \citep{shao2024deepseekmath}. 
Then a line of highly influential works, including DeepSeek-Math~\citep{shao2024deepseekmath}, Deepseek-prover~\citep{xin2024deepseekprover}, and DeepSeek-R1~\citep{guo2025deepseek}, sparked a wave of follow-up refinements to GRPO-style post-training.
More recent variants, such as Dr.\ GRPO~\citep{liu2025understanding}, DAPO~\citep{yu2025dapo}, and
GRPO-RoC~\citep{shang2025rstar2agent} 
, further refine GRPO by correcting optimization bias and improving token efficiency, alleviating gradient dead zones and entropy collapse, introducing agentic setting with resampling-on-correct trajectories.


During GRPO training, prompts may fail to yield enough non-degenerate groups with non-zero success rates. To enhance sample efficiency and training stability, recent adaptive approaches first estimate prompt informativeness (e.g., via pre-rollouts or lightweight difficulty signals), and then adjust the rollout/training pipeline via selective prompt/rollout selection, reuse of responses or trajectories across iterations (e.g., via replay buffers), or post-hoc filtering and pruning of low-value samples \citep{zhang2025speed, sun2025improving, liao2025enhancing, kong2025rethinking, li2025knapsack, bamba2025xrpo}.



\section{Preliminaries}




\subsection{Background: GRPO Algorithm}


The core mechanism of GRPO involves estimating the advantage of a sampled response relative to the performance of its peers within a group. For a given prompt, the policy generates a group of $N$ responses (rollouts), and each response $o_i$ receives a scalar reward $r_i$. The set of rewards for the group is denoted as $\mathbf{r} = \{r_1, r_2, \ldots, r_N\}$. The advantage $\hat{A}_i$ for the $i$-th response is then calculated by normalizing its reward against the group's mean and standard deviation:
\begin{equation}
    \hat{A}_i = \frac{r_i - \text{mean}(\mathbf{r})}{\text{std}(\mathbf{r})}
\end{equation}
GRPO optimizes a group-level clipped objective function to update the policy parameters $\theta$. A key simplification is that all tokens within a single response $o_i$ share the same advantage value $\hat{A}_i$, as the reward is given at the response level. The objective function $J(\theta)$ is defined as:


\begin{equation}
\label{obj_grpo}
\begin{aligned}
    &J(\theta) = \mathbb{E}_{q, \{o_i\}} \Bigg[ \frac{1}{N} \sum_{i=1}^{N} \frac{1}{|o_i|} \sum_{t=1}^{|o_i|} \min \Big( \rho_{i,t} \hat{A}_i, \\
    &\quad \text{clip}(\rho_{i,t}, 1-\epsilon, 1+\epsilon)\hat{A}_i \Big) - \beta D_{KL}(\pi_{\theta} || \pi_{\text{ref}}) \Bigg]
\end{aligned}
\end{equation}
where $\rho_{i,t}$ is the importance sampling ratio for the $t$-th token of the $i$-th response, defined as:
$\rho_{i,t} = \frac{\pi_\theta(o_{i,t} | q, o_{i,<t})}{\pi_{old}(o_{i,t} | q, o_{i,<t})}$.
Here, $\pi_\theta$ is the current policy being optimized, and $\pi_{old}$ is the frozen policy from which the rollouts were sampled. The $\text{clip}$ function constrains the policy ratio $\rho_{i,t}$ within the range $[1-\epsilon, 1+\epsilon]$, which stabilizes training by preventing excessively large policy updates. Finally, the KL-divergence term $D_{KL}(\pi_\theta || \pi_{ref})$ acts as a regularizer, penalizing the updated policy $\pi_\theta$ for deviating too far from a reference policy $\pi_{ref}$ (often the initial SFT model).

Starting from the GRPO objective \eqref{obj_grpo}, and under the standard small-step (trust-region) regime where the policy ratios remain within the clipping range so that no clipping is activated (i.e., $\rho_{i,t}\le 1+\epsilon$ for all $t$), the gradient simplifies to
\begin{equation}
\label{eq:simplified-grad}
\begin{aligned}
&\nabla_\theta J(\theta)
= \mathbb{E}_{q,\{o_i\}}\!\Bigg[
\frac{1}{N}\sum_{i=1}^{N}\frac{\hat A_i}{|o_i|}
\sum_{t=1}^{|o_i|}
\\[-2pt]
&
\rho_{i,t}\nabla_\theta \log\pi_\theta(o_{i,t}\!\mid\! q,o_{i,<t})
\Bigg]
- \beta\,\nabla_\theta D_{\mathrm{KL}}(\pi_\theta\Vert \pi_{\mathrm{ref}}).
\end{aligned}
\end{equation}
The proof can be found in \ref{sup_proof_policy_gradient}.

While highly effective, GRPO's reliance on a fixed number of rollouts for every problem introduces inefficiencies, some problems result in all-fail or all-correct cases, where zero-advantages ($\hat A_i = 0$) leads to a shrinking batch-gradient magnitude. A substantial portion of the computation and training time is wasted on zero-advantage queries, making the process inefficient.

\subsection{The Impact of Zero-Advantage Problems on GRPO}

To dive deep in the inefficiency caused by zero-advantage problems, 
We hypothesize that performance is disproportionately affected by problems for which the model fails to produce any correct roll-outs termed \textbf{Zero-Accuracy Problems}.

When all rollouts for a given problem fail (i.e., the rewards $r_1, \dots, r_N = 0$), the empirical success rate $u = \text{mean}(\mathbf{r}) = 0$, leading to a complete collapse of the advantage, 
\begin{equation}
    \hat{A}_i = \frac{r_i - u}{\text{std}(\mathbf{r})} = \frac{0 - 0}{\text{std}(\mathbf{r})} = 0.
\end{equation}

To validate this hypothesis, we compare two GRPO variants---a stronger and a naive baseline across multiple training checkpoints. At each step $t$, we track two metrics:
(1) \textbf{zero-accuracy problem ratio}, $x_t^{\text{zero}}$: the fraction of problems whose rollouts are all incorrect:
    \begin{equation}
        x_t^{\text{zero}} = \frac{\sum_{p=1}^{P} \mathbb{I}(\text{All } r_i=0 \text{ for problem } p)}{P}.
    \end{equation}
    where $P$ is the number of problems in the batch and $\mathbb{I}(\cdot)$ is the indicator function.

(2) \textbf{Normalized model test score}, $\tilde{y}_t$: the relative test performance between two methods at step $t$. Let the raw test scores for the two methods, computed as
\begin{equation}
\tilde{y}_{t,i} = \frac{y_{t,i}}{\max(y_{t,1}, y_{t,2})},
\end{equation}
where $y_{t,i}$ denotes the raw test score of method $i$. This normalization removes the effect of overall score scaling during training.

\begin{wrapfigure}{h}{0.5\columnwidth}
    \vspace{-1.5em} 
\includegraphics[width=\linewidth]{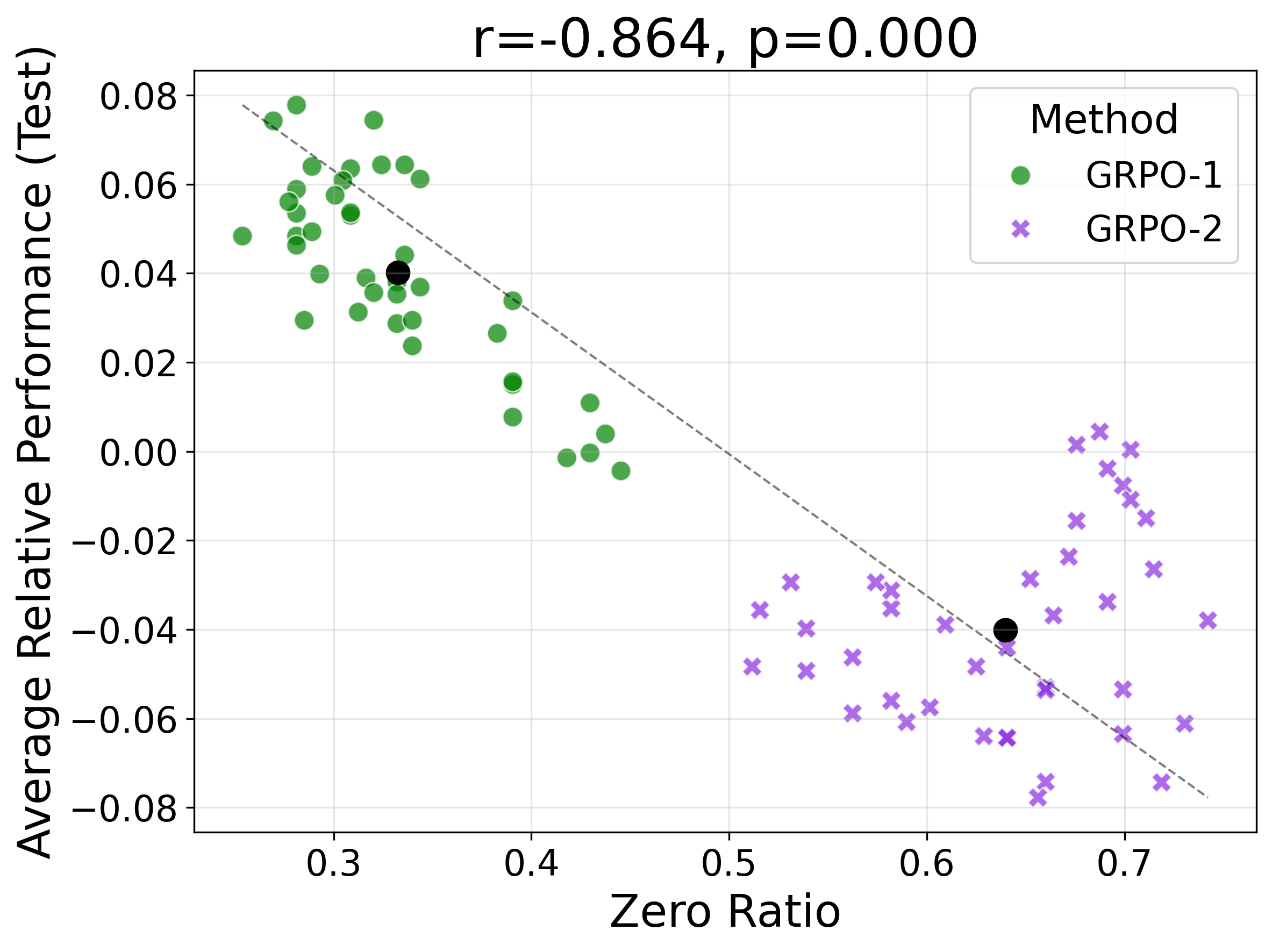}
    \caption{Correlation between zero-accuracy ratio and relative model performance ($r=-0.86, p<0.001$).}
    \label{fig:correlation_plot}
    \vspace{-.5em} 
\end{wrapfigure}

We compute the Pearson correlation between the zero-accuracy ratio $x^{\text{zero}}$ and the relative performance score $y$ across all checkpoints. As shown in Figure~\ref{fig:correlation_plot}, the two are strongly and significantly negatively correlated ($r = -0.86, p < 0.001$). Thus, a higher proportion of zero-accuracy problems  predicts lower model performance.


For all-correct queries, just like the zero-accuracy ones, face the gradient dead zone issue. Theorem~\ref{thm:grad_norm_maximization} formalizes this phenomenon, showing that the gradient norm vanishes when all rollouts for a query are identical, explaining why GRPO’s fixed-rollout scheme wastes compute on uninformative queries.




\begin{figure*}[t]
    \centering
     \setlength{\dbltextfloatsep}{10pt minus 2pt}
    \includegraphics[width=0.84\textwidth, trim=0 0 0 10, clip]{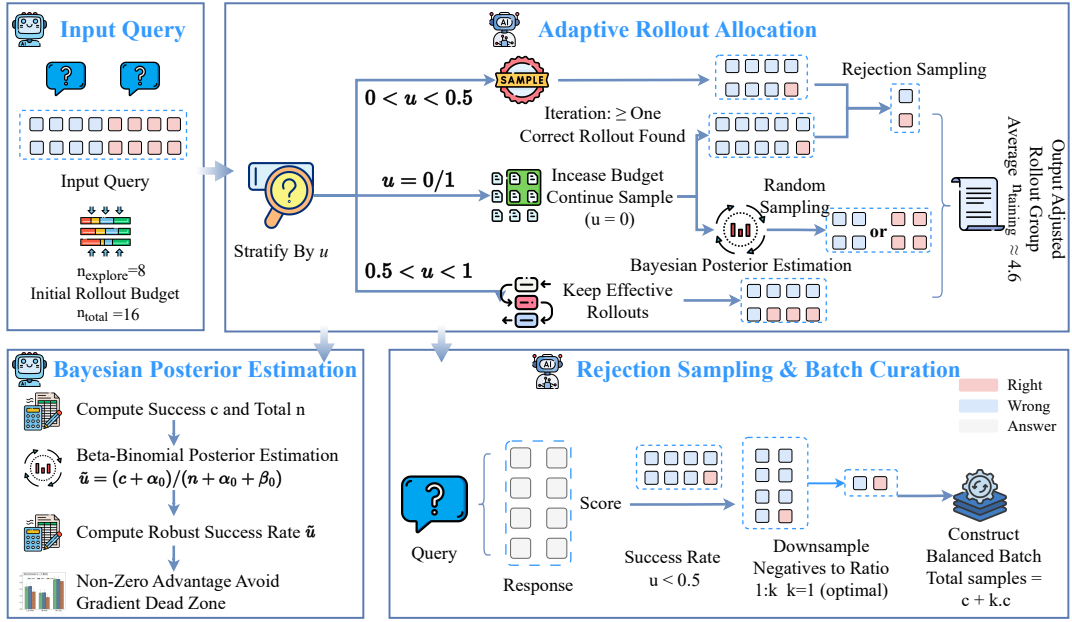}
    \caption{High-level architecture of the Posterior-Guided Sampling (AERO) framework.
AERO begins with $n_{\text{total}} = 16$ total rollouts per query, from which $n_{\text{explore}} = 8$ samples are used for stratified exploration based on success rate $u$. Finally, AERO yields an effective training size of approximately $n_{\text{training}}=4.6$ rollouts per query.
    }
    \label{fig:pgs_framework}
\end{figure*}

\section{Methodology - The AERO Framework}
AERO restructures rollout allocation into a two-stage process, as shown in Figure~\ref{fig:pgs_framework}. In Stage I, a small probe budget stratifies queries based on correctness. In Stage II, differentiated strategies are applied:
(1) The partial-correct low accuracy queries 
are rescued via iterative resampling, with early exit upon first success. (2) Zero-advantage queries are mitigated via a Bayesian posterior over success rate, ensuring stable gradients rather than discarding them entirely. (3) Partially correct queries undergo fine-grained rejection sampling: correct rollouts are retained, while incorrect ones are down-sampled to achieve a balanced positive-to-negative ratio.


\subsection{Stage I: Exploration \& Stratification.} 
Given a total budget of $n$ (e.g., 16) rollouts per problem, AERO first allocates a partial exploration budget, $n_{explore}$ (e.g., 8), to generate an initial set of responses. Based on the empirical success rate $u$ from this initial pass, problems are stratified into three distinct groups corresponding to their perceived correctness rate: 
\emph{low-success partial} ($0\leq u<0.5$), and \emph{high-accuracy} ($1 \geq u\ge 0.5$).

\begin{algorithm}[t]
\caption{\footnotesize Posterior Guided Sampling (AERO)}
\label{alg:pgs}
\SetAlgoLined
\KwIn{Batch $\{q_i\}$, budgets $n_{total}, n_{explore}$, increment $n_{extra}$, ratio $k$, rescusing threshold $S$.}
\KwOut{Curated set $\mathcal{D}_{train}$.}

\textbf{Stage 1: Exploration} \\
Generate $n_{explore}$ rollouts; update priors and success rate. \\
Let $S \in [0.5n_{explore}]$, $u_{res} = S/n_{explore}$
Let $\mathcal{Q}_{res} \!=\! \{q_i \mid u_i \leq u_{res}\}$, 
$\mathcal{Q}_{partial} \!=\! \{q_i \mid u_{res}<u_i<0.5\}$, 
$\mathcal{Q}_{high} \!=\! \{q_i \mid u_i \ge 0.5\}$.
\\ 
\textbf{Stage 2: Exploitation} \\
\Indp
$\mathcal{Q}_{partial}$: keep all correct + $k$-scaled incorrect. \\
$\mathcal{Q}_{high}$: keep all rollouts. \\
$\mathcal{Q}_{res}$: iterative rescuing with $n_{extra}$. \\
\Indm
\end{algorithm}

\subsection{Stage II: Differentiated Exploitation.} The remaining budget, $N_{budget} = (n - n_{explore}) \times P$, is then dynamically allocated. Each problem stratum enters a specialized processing pipeline designed to maximize its contribution to the policy gradient. This targeted approach ensures that computational resources are focused where they are most needed, either by attempting to rescue difficult problems or by curating an efficient batch from easier ones. The complete AERO process for a single training step is formally detailed in Algorithm~\ref{alg:pgs}. The algorithm outlines the initial exploration pass, the stratification logic, and the subsequent differentiated exploitation strategies for each problem category before curating the final training batch for the GRPO update.

\subsubsection{Strategy 1: Iterative Rescue of Rescue-Threshold Problems}

Let $S \in \mathbb{N}_0$ be a hyperparameter, with $S \in [0, 0.5n_{explore}]$. 
It defines a threshold on the number of correct rollouts, corresponding to a success-rate threshold $u_{res} = S / n_{explore}$. 
Problems with a success rate less than or equal to $u_{res}$ are classified as the rescue set $\mathcal{Q}_{res}$. 
Problems in the $\mathcal{Q}_{res}$ stratum, including zero-accuracy ones which provide no learning signal under standard GRPO, enter a budget-aware \textbf{iterative rescue phase} (Figure~\ref{fig:pgs_framework}). 
Instead of expending the entire remaining budget at once, these problems receive a small number of additional rollouts ($n_{extra}$, e.g., 2) per iteration. This loop continues until either the first successful rollout is generated 
(at least one of rollout in the $n_{extra}$ is correct)
or the total budget $N_{budget}$ is exhausted. 
\begin{wrapfigure}[10]{r}{0.5\columnwidth}
  \centering
  \includegraphics[width=\linewidth]{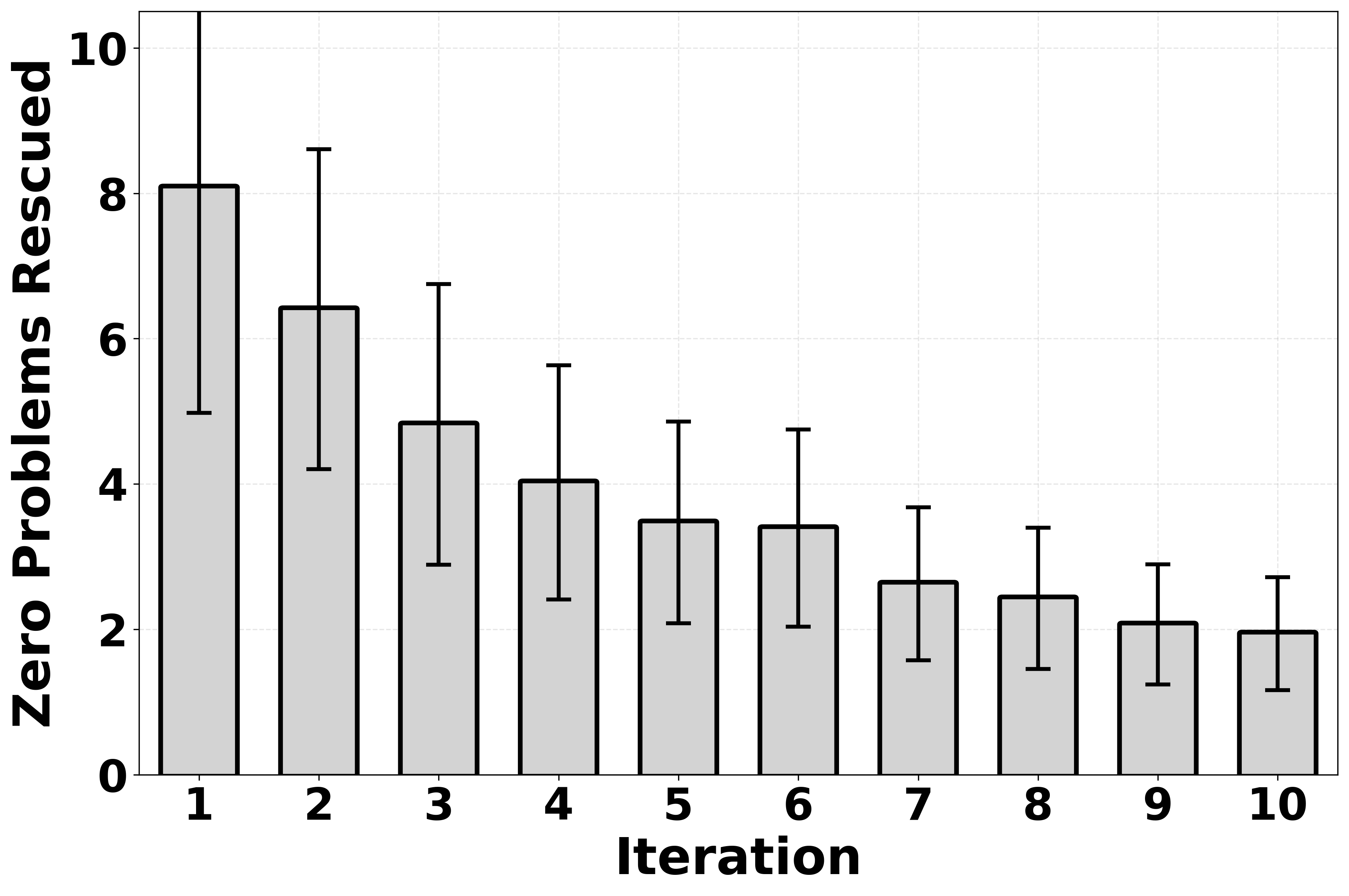}
  \caption{Rescued problems across iterations.}
  \label{fig:rescued_over_iters}
\end{wrapfigure}
In our study, we set $S=0$, so that rescue-threshold problems coincide with zero-accuracy problems, which aligns best with our model and dataset characteristics. In general, $S$ can be adjusted to match model capability and task difficulty: the more powerful the model and the easier the query, the higher $S$ may be set. 
Figure~\ref{fig:rescued_over_iters} shows how our method
gradually moves zero-accuracy problems into the non-zero pool across iterations. This strategy efficiently converts zero-accuracy problems into valid training samples, increasing the proportion of non-zero-accuracy problems in the batch by 26\% in our experiments.

\par\medskip
\subsubsection{
Strategy 2: Robust Estimation via Bayesian Posteriors}
Even after rescue, certain zero-accuracy queries remain in the zero-advantage regime. Likewise,
queries in the 
all-correct ($u=1$) stratum, suffer from the same "gradient dead zone" as zero-accuracy problems, as their advantage also becomes zero. To combat this, we replace their the empirical mean $u$ with a more robust \textbf{Bayesian posterior estimate}, $\tilde{u}$. We model the success rate with a Beta-Binomial distribution, which is a conjugate prior. Starting with a weak prior $p(u) \sim \text{Beta}(\alpha_0, \beta_0)$, where in practice we set $\alpha_0=\beta_0=1$, after observing $c$ successes in $n$ rollouts, the posterior becomes $p(u|c,n) \sim \text{Beta}(\alpha_0+c, \beta_0+n-c)$. We then use the mean of this posterior distribution:
\begin{equation}
    \tilde{u} = \frac{c + \alpha_0}{n + \alpha_0 + \beta_0}
\end{equation}
The Bayesian estimate ensures that $\tilde{u}$ is never exactly 0 or 1, guaranteeing a stable, non-zero advantage $\hat{A}_i \neq 0$ for all problems and a consistent learning signal. Figure~\ref{fig:gradient_norm_vs_pg} valides the Bayesian's work. 
After Bayesian posterior estimation provides a stable learning signal, we retain only a small curated subset (typically $4$ rollouts from $n$) via random sampling.


\subsubsection{Strategy 3: Efficient Batch Curation via Rejection Sampling}
For low-success problems in $\mathcal{Q}{\text{low}}$, where $u_{res} < u < 0.5$, using all $n_{\text{explore}}$ rollouts may result in a batch dominated by incorrect samples. We apply \textbf{rejection sampling} to construct a smaller, more information-dense batch (Figure~\ref{fig:pgs_framework}). We retain all $c$ correct rollouts and downsample the $n_{\text{explore}} - c$ incorrect ones, selecting only the top-$m$ most informative samples. The number of selected incorrect rollouts is set by $m = k \cdot c$, where $k$ is a tunable ratio. This creates a balanced $1:k$ mix of correct and incorrect rollouts, reducing computational cost.

\subsection{Theoretical Guidance for Rejection Sampling}
While rejection sampling is empirically effective, the choice of the downsampling ratio $k$ can be theoretically motivated. We derived the Theorem~\ref{thm:grad_norm_maximization} concerning the upper bound of squared norm of the GRPO policy gradient, which provides guidance for selecting an optimal $k$. 


\begin{theorem}[GRPO Gradient Norm]
\label{thm:grad_norm_maximization}
Let a training subset contain $n = c + m$ rollouts, with $c$ correct rollouts ($r_i = 1$) forming set $F_C$ and $m$ incorrect rollouts ($r_i = 0$) forming set $F_S$.  
Under the simplifying assumptions of uniform rollout length $L$ and zero KL penalty ($\beta = 0$), the group-relative policy gradient in Eq.~\eqref{eq:simplified-grad} is
\begin{equation}
\label{eq:gradient_for_theorem}
G(\theta)
= \mathbb{E}_{q,\{o_i\}}\!\left[
\frac{1}{NL}\!
\sum_{i=1}^{N}\!
\sum_{t=1}^{L}
\hat{A}_i\,\rho_{i,t}\,
\nabla_\theta \log\pi_\theta(o_{i,t}\!\mid q,o_{i,<t})
\right],
\end{equation}
where $\rho_{i,t} = \frac{\pi_\theta(o_{i,t}\mid q,o_{i,<t})}{\pi_{\text{old}}(o_{i,t}\mid q,o_{i,<t})}$.  
Then its squared norm satisfies
\[
\|G(\theta)\|^2
\le
\frac{c\,m}{(c+m)^2}
\,\mathbb{E}\!\left[\|\bar{v}_C - \bar{v}_F\|^2\right],
\]
where
$\bar{v}_C = \frac{1}{c}\sum_{j \in F_C} v_j$,
$\bar{v}_F = \frac{1}{m}\sum_{j \in F_S} v_j$,
and
$v_j = \frac{1}{L}\sum_{t=1}^{L}\nabla_\theta \log\pi_\theta(o_{j,t}\mid q,o_{j,<t})$.  
\end{theorem}

The proof is in the in the Supplementary~\ref{sup_proof_thm_grad_norm_maximization}. 
The scalar factor $\tfrac{c\,m}{(c+m)^2}$ reflects the composition of correct and incorrect rollouts, while $\|\bar{v}_C - \bar{v}_F\|^2$ measures the discrepancy between their mean policy-gradient directions. This theorem reveals that the quantity of the scalar term is related to the allocation of $c$ and $m$ ($k = 1$). Below Lemma~\ref{lma:grad_norm_maximization} states this fact. 
\begin{lemma}[Maximization of the GROP Gradient Norm]
\label{lma:grad_norm_maximization}
For the same set up in Theorem \ref{thm:grad_norm_maximization},
\[
\|G(\theta)\|^2 \le \frac{c \cdot m}{(c+m)^2} \cdot \mathbb{E}\|\bar{v}_C - \bar{v}_F\|^2,
\] the balance term $\frac{c \cdot m}{(c+m)^2}$ is maximized if and only if $m=c$.
\end{lemma}
 This occurs when $m=c$, which corresponds to a downsampling ratio of $k=1$. The proof is in the Supplementary~\ref{sup_proof_lemma_grad_norm_maximization}. 

\subsection{Computation Cost Saving Across All Strategies}
Across all three strategies, AERO dynamically reduces effective training size, 
this results in an average of $n_{\text{training}}\approx4.6$ rollouts per query, compared to $16$ in GRPO, explaining the observed $\sim$50\% wall-clock speedup.

In strategy 1, for rescue-threshold queries, the iterative rescue loop terminates immediately once the first successful rollout is found, yielding a group size to $2$ (one correct, one incorrect), rather than $n=16$. In strategy 2, for zero-advantage queries (all-fail or all-correct), we retain only $4$ rollouts after Bayesian stabilization. In strategy 3, for low-success partial queries,  we use rejection sampling to construct a balanced $1\!:\!k$ mix of correct and incorrect rollouts, instead of using all $16$ rollouts.

Empirically, we find that early termination on rescued queries and downsampling on zero-advantage accounts for roughly 29\% of the total speedup, 
while the remaining $21\%$ arises from rejection sampling on low-success partial queries.  We conduct an ablation study to quantify each strategy's contribution to the total speedup (see Section~\ref{sec:strategy_ablation} and Figure~\ref{fig:strategy_contribution}).

\begin{figure}[t]
    \centering
    \includegraphics[width=\linewidth]{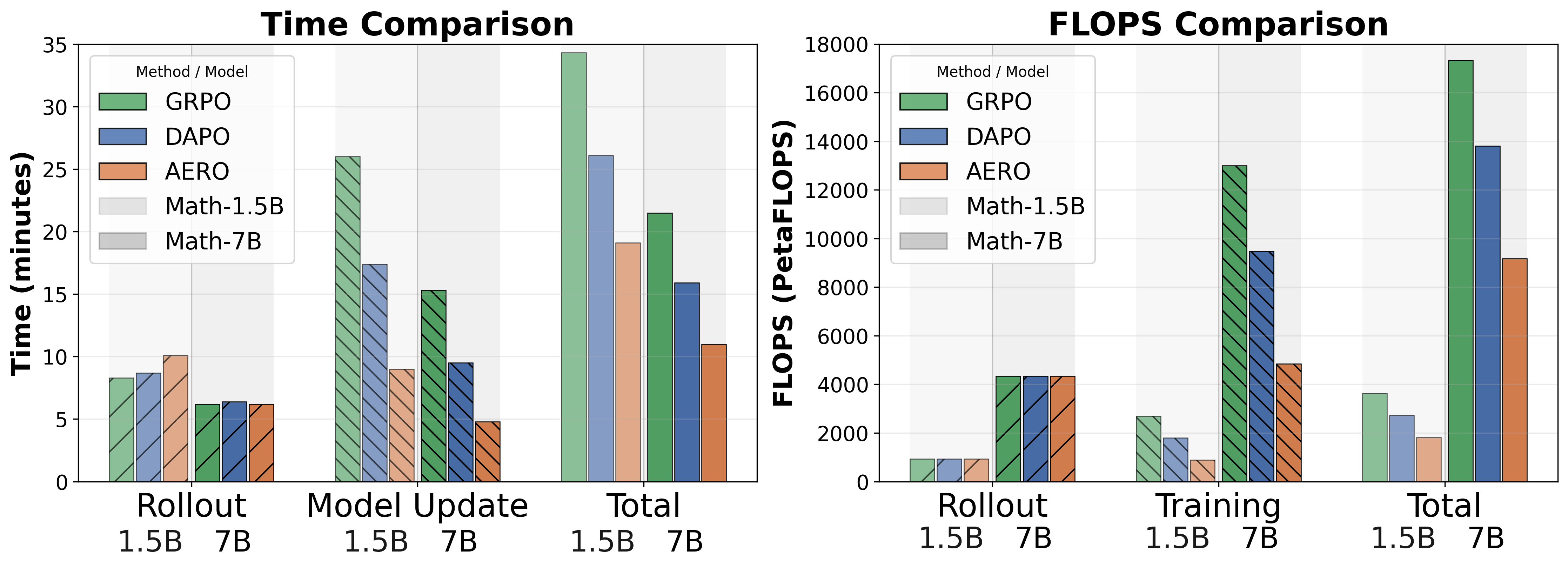}
    \caption{Comparison of per-step (a) wall-clock latency and (b) computational cost (FLOPS), 
decomposed into Data Generation (Rollout) and Model Update (Training) and Total (Rollout+Training).}
    \label{fig:efficiency_charts}
    \vspace{-1.0em}
\end{figure}

\section{Experiments}

We conduct a series of experiments to evaluate the effectiveness of our proposed AERO framework. We aim to answer three primary questions: (1) Does AERO offer tangible improvements in computational and time efficiency?  (2) Does AERO outperform standard baselines in reasoning tasks? (3) What is the contribution of each core component of the AERO framework?

\subsection{Experimental Setup}

\textbf{Core Model \& Hardware.}
All experiments are conducted using the \textbf{Qwen2.5-Math-1.5B}, \textbf{Qwen2.5-1.5B-Base}, \textbf{Qwen2.5-Math-7B}, and \textbf{Qwen2.5-7B-Instruct-1M} models. 
Training for the 1.5B models is performed on \textbf{2$\times$ NVIDIA A10 GPUs}, while the 7B models are trained on \textbf{4$\times$ NVIDIA A100 GPUs}.

\textbf{Datasets.}
For the \textbf{Math} domain, training uses the OpenR1-AERO,  a 40k curated set of \texttt{OpenR1-Math-220k} dataset including mathematical problems released on Hugging Face, 2025 \citep{zhao20251}. 
Validation and testing are conducted on held-out sets covering five reasoning sub-domains: \textbf{AIME}, \textbf{AMC}, \textbf{MATH}, \textbf{Minerva}, and \textbf{Olympiad} for broader evaluation. For the \textbf{Code} domain, training uses the dataset of 1.4k coding problems introduced by \citet{cui2025process}. 
Validation and testing use 0.5k held-out sets from \textbf{APPS} \citep{hendrycks2021measuring}, \textbf{CodeContests} \citep{li2022competition}, and \textbf{TACO} \citep{li2023taco}.

\textbf{Training Protocol.}
For each method, we select the single best-performing checkpoint based on the highest average score achieved on the validation set. This selected checkpoint is then used for a final, one-time evaluation on the held-out test set to report our main results.

\textbf{Baselines and Method Configuration.}
We compare AERO against two strong baselines, ensuring a fair comparison by allocating the same total generation budget to each method.
(1) \textbf{GRPO (Main Baseline):} The standard GRPO algorithm is run with a fixed budget of $n=16$ rollouts for every problem.
(2) \textbf{DAPO:} We compare against the DAPO algorithm, which incorporates techniques like clipping and adaptive sampling. It also uses a total generation budget of $n=16$.
(3) \textbf{AERO (Ours):} Our method is configured with a total budget of $n_{total} = 16$. This is split into an exploration budget of $n_{explore} = 8$ for Stage 1, leaving a remaining budget of $8$ rollouts for the differentiated exploitation stage. We use a downsampling ratio of $k=1$ for rejection sampling and a weak Bayesian prior of $\text{Beta}(\alpha_0=1, \beta_0=1)$ for zero-advantage cases.

\subsection{Main Results}

\paragraph{Compute and Time Performance.}

{\captionsetup[sub]{font=tiny}
\begin{figure}[t]
\centering
\vspace{-0.2em}

\begin{subfigure}{0.15\textwidth}
    \centering
    \includegraphics[width=\linewidth]{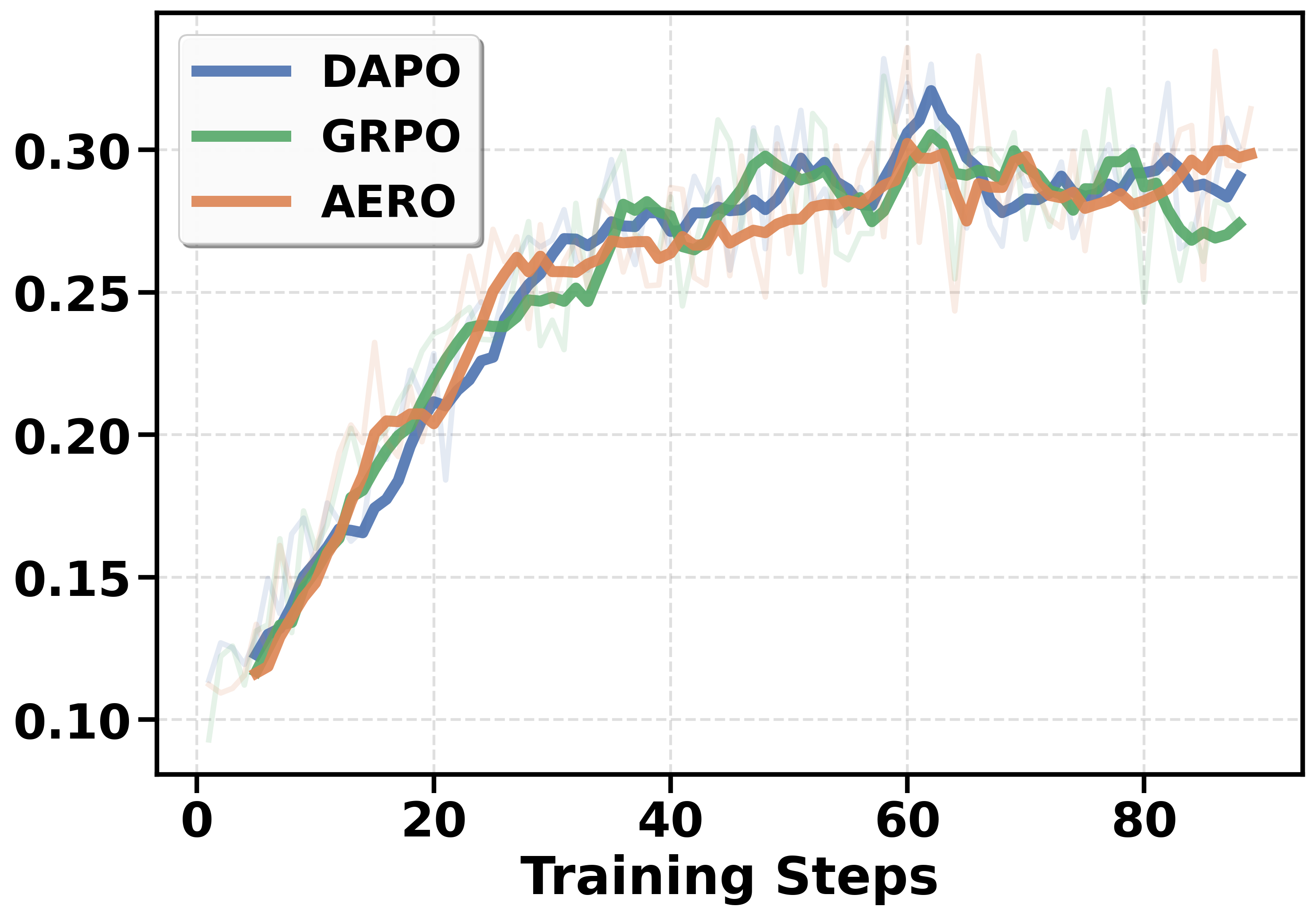}
    \caption{Training Reward\\\hphantom{Problems Ratio}}
    \label{fig:training_rewards}
\end{subfigure}\hfill
\begin{subfigure}{0.15\textwidth}
    \centering
    \includegraphics[width=\linewidth]{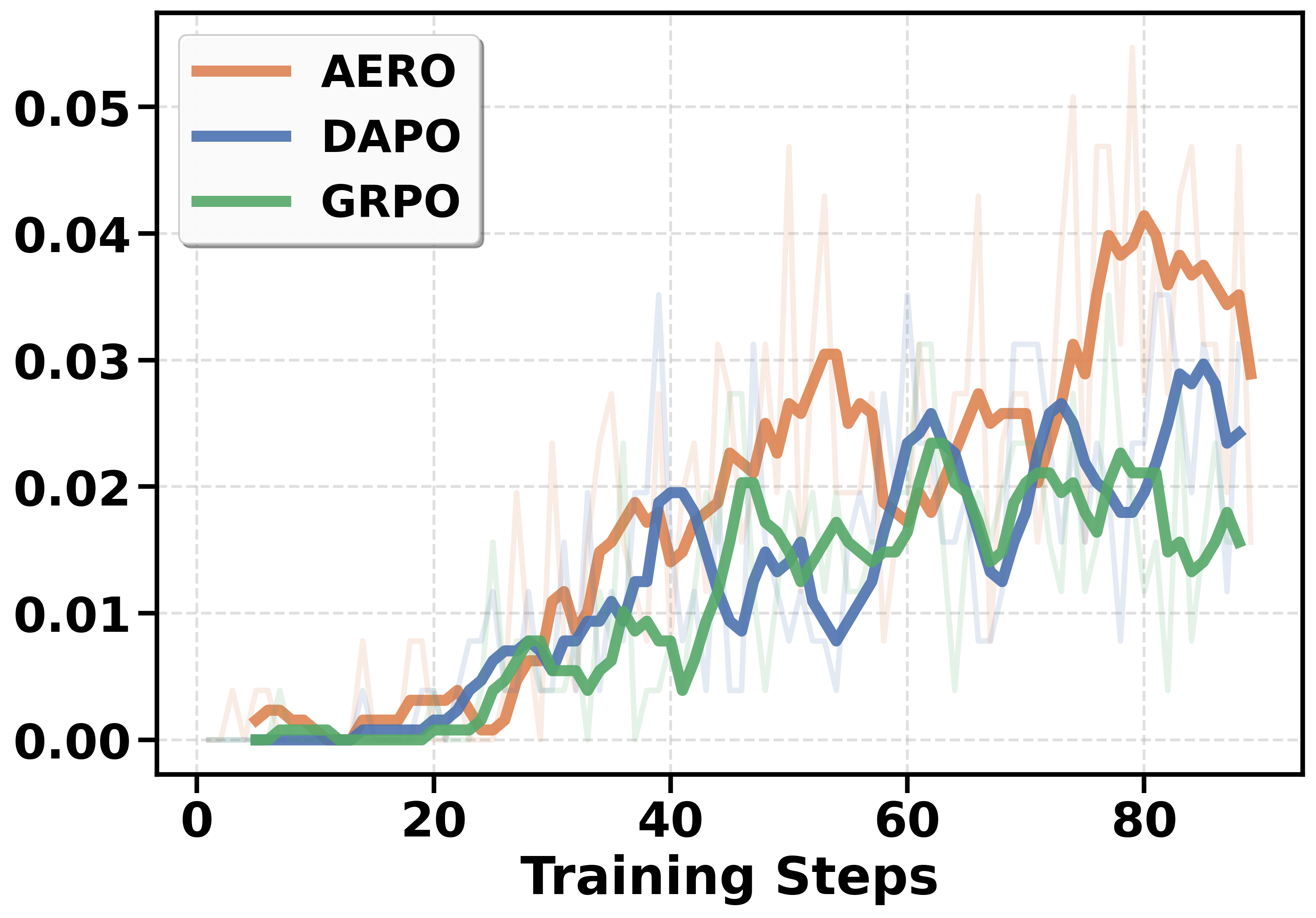}
    \caption{All-Correct Problems Ratio}
    \label{fig:all-correct}
\end{subfigure}\hfill
\begin{subfigure}{0.15\textwidth}
    \centering
    \includegraphics[width=\linewidth]{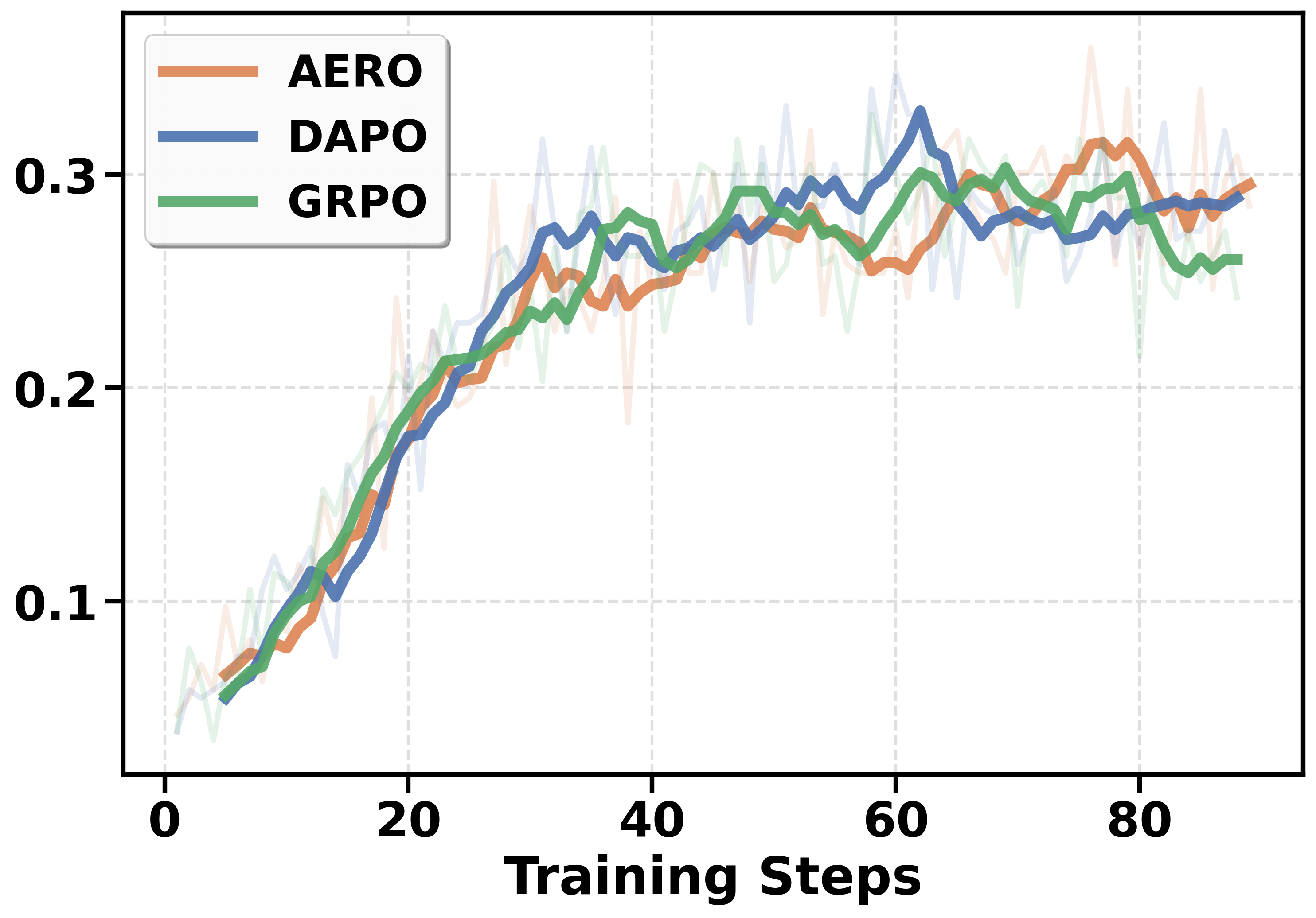}
    \caption{High-Success Partial Problems Ratio}
    \label{fig:high-partial}
\end{subfigure}

\vspace{-0.2em}
\caption{Training performance on the 1.5B model.}
\label{fig:training_performance_1p5b}
\end{figure}
}

\begin{table*}[t] 
\centering
\caption{Main results. Per step efficiency (Training Time and FLOPs)  and Held-out Test Set Performance (Avg@8, Pass@8 across five reasoning benchmarks).
Best results are in \textbf{bold}.}
\vspace{5pt}
\label{tab:main_results_combined}
\resizebox{\textwidth}{!}{%
\begin{tabular}{@{}lcc|cccccc|cccccc}
\toprule
& \textbf{Time (min)} & \textbf{FLOPs} & \multicolumn{6}{c}{\textbf{Avg@8}} & \multicolumn{6}{c}{\textbf{Pass@8}} \\
\cmidrule(lr){2-3} \cmidrule(lr){4-9} \cmidrule(lr){10-15}
\textbf{Method} &  &  & \textbf{aime} & \textbf{amc} & \textbf{math} & \textbf{minerva} & \textbf{olympiad} & \textbf{Total (P-value)}  &
\textbf{aime} & \textbf{amc} & \textbf{math} & \textbf{minerva} & \textbf{olympiad} & \textbf{Total} \\
\midrule
\multicolumn{15}{l}{\textbf{Qwen2.5-1.5B-Math}} \\
\midrule
AERO  & \cellcolor{table-blue!66} \textbf{19.1} & \cellcolor{table-blue!66} \textbf{1811} & 
0.107 & 0.323 & 0.578 & 0.171 & 0.256 & \cellcolor{table-blue!66} \textbf{0.3452} & 
0.286 & 0.660 & 0.863 & 0.409 & 0.517 & \cellcolor{table-blue!66} \textbf{0.6121}  \\
GRPO & 34.3 & 3627 & 
0.083 & 0.326 & 0.571 & 0.173 & 0.251 & 0.3401 ($=0.28$)& 
0.238 & 0.604 & 0.856 & 0.421 & 0.519 & 0.6090  \\
DAPO & 26.1 & 2722 & 
0.101 & 0.330 & 0.561 & 0.171 & 0.257 & 0.3405 ($=0.31$)& 
0.333 & 0.642 & 0.843 & 0.409 & 0.526 & 0.6100  \\
Base & --- & ---  & 0.025 & 0.192 & 0.322 & 0.085 & 0.129 & 0.1617 ($$<1\text{e-6}$$) & 0.200 & 0.700 & 0.778 & 0.345 & 0.399 & 0.5050   \\
\midrule
\multicolumn{15}{l}{\textbf{Qwen2.5-7B-Math}} \\
\midrule
AERO & \cellcolor{table-blue!66} \textbf{10.8} & \cellcolor{table-blue!66} \textbf{9181} & 
0.214 & 0.552 & 0.719 & 0.289 & 0.382 & \cellcolor{table-blue!66} \textbf{0.4800} & 
0.476 & 0.811 & 0.905 & 0.503 & 0.626 & \cellcolor{table-blue!66} \textbf{0.7000}  \\
GRPO & 21.5 & 17334 & 
0.208 & 0.543 & 0.714 & 0.267 & 0.372 & 0.4690 (=0.05) & 
0.476 & 0.721 & 0.902 & 0.491 & 0.607 & 0.6860 \\
DAPO & 15.8 & 13743 & 
0.208 & 0.533 & 0.720 & 0.281 & 0.376 & 0.4750 (=0.29)& 
0.476 & 0.736 & 0.905 & 0.491 & 0.619 & 0.6910   \\
Base & --- & --- & 
0.138 & 0.358 & 0.539 & 0.179 & 0.196 & 0.3110 ($<1\text{e-6}$) & 
0.400 & 0.700 & 0.876 & 0.434 & 0.510 & 0.6220  \\
\bottomrule
\end{tabular}%
}
\end{table*}

We analyze the computational efficiency of each method by decomposing both the wall-clock time and the total FLOPS per step into their Data Generation (Rollout) and Model Update (Training) components. As shown in Figure~\ref{fig:efficiency_charts}, AERO demonstrates significant efficiency gains.

 AERO runs \textbf{1.8x faster per step} on the 1.5B model (19.1 vs. 34.3 minutes)  and \textbf{1.9x faster per step} on the 7B model (10.8 vs. 21.5 minutes) than the GRPO baseline. This speedup is driven by a 61$\%$ reduction in the model update time for 1.5B and 68$\%$ for 7B.
This efficiency is also reflected in computational cost. While Rollout FLOPS are identical for all methods, AERO requires \textbf{2.8x fewer} Training FLOPs than GRPO on the 1.5B model and \textbf{2.7x fewer} on the 7B model. Consequently, its total computational cost per step is reduced by \textbf{2.0x} and \textbf{1.9x}, respectively. AERO improves the end-to-end training efficiency. For zero-accuracy problems ($u=0$), AERO yields the lowest ratio$-$\textbf{0.261} vs. 0.336 (GRPO) and 0.345 (DAPO) on the 1.5B model, and \textbf{0.178} vs. 0.244 (GRPO) and 0.249 (DAPO) on the 7B model, indicating more effective use of the rollout budget. 

\textbf{Training Performance.}
The training reward curves (Figure~\ref{fig:training_rewards}) show that AERO demonstrates a stable learning trajectory, closely matching the performance of the GRPO and DAPO baselines throughout the training process.  In addition, Figure~\ref{fig:all-correct} shows that AERO consistently yields a higher proportion of All-correct  problems compared to GRPO and DAPO. Meanwhile, as shown in Figure~\ref{fig:high-partial}, all methods achieve a comparable ratio of High-Success Partial Problem.

\textbf{Testing Performance.} The primary results on the held-out test set are presented in Table~\ref{tab:main_results_combined}. AERO achieves an average Avg@8 score of \textbf{0.3451}, outperforming GRPO (0.3401) and DAPO (0.3405) on the 1.5B model, and \textbf{0.4800}, surpassing GRPO (0.4690) and DAPO (0.4750) on the 7B model. A similar trend is observed in the Pass@8 metric, where AERO scores \textbf{0.6121} on 1.5B and \textbf{0.7} on 7B.

\begin{table}[!h]
\centering
\caption{Computational budget and alternative sampling strategies for Qwen2.5-1.5B-Math. 
The best-performing method is highlighted in blue, and time-matched baselines for comparison are in light blue. Results generalize to other model sizes and configurations as shown in 
\S~\ref{app:additional_budget_baselines}.}
\vspace{2pt}
\label{tab:budget_baselines_main}
\begin{icmltable}
\adjustbox{max width=\columnwidth}{%
\begin{tabular}{@{}lcccc@{}}
\toprule
\textbf{Method} & \textbf{Time (min)} & \textbf{FLOPs (Total)} & \textbf{Avg@8} (P-value) & \textbf{Pass@8} \\
\midrule
GRPO & 34.3 & 3627 & 0.3401 ($=0.28$) & 0.6090 \\
DAPO & 26.1 & 2722 & 0.3405 ($=0.31$) & 0.6100 \\
AERO  & \cellcolor{table-blue!66} \textbf{19.1} & \textbf{1811} & \cellcolor{table-blue!66} \textbf{0.3452} & \cellcolor{table-blue!66} \textbf{0.6121} \\
N/2 & 23.7 & 1863 & 0.3211 ($=0.004$) & 0.5954 \\
0.32N &\cellcolor{table-blue!30} 20.1 & 1133 & \cellcolor{table-blue!30} 0.3184 ($=0.002$) & \cellcolor{table-blue!30} 0.5891 \\
N/4 & 16.3 & 906 & 0.3133 ($=0.001$) & 0.5870 \\
N/8 & 8.9 & 453 & 0.2820 ($<1\text{e-4}$) & 0.5791 \\
Base & --- & --- & 0.1617 ($<1\text{e-6}$) & 0.5050 \\
\bottomrule
\end{tabular}%
}  \end{icmltable}
\vspace{-.6em}
\end{table}

\subsubsection{Additional Baselines}

To better compare the compute and performance trade-off with AERO, we evaluate GRPO under Reduced-Budget Rollouts. We uniformly reduce the rollout budget per query from the default $N=16$ to smaller values ($N/2$,$0.32N$,$N/4$,$N/8$).
     

\begin{figure}[H] 
    \centering
    \vspace{-0.6em}
    \includegraphics[width=0.9\linewidth]{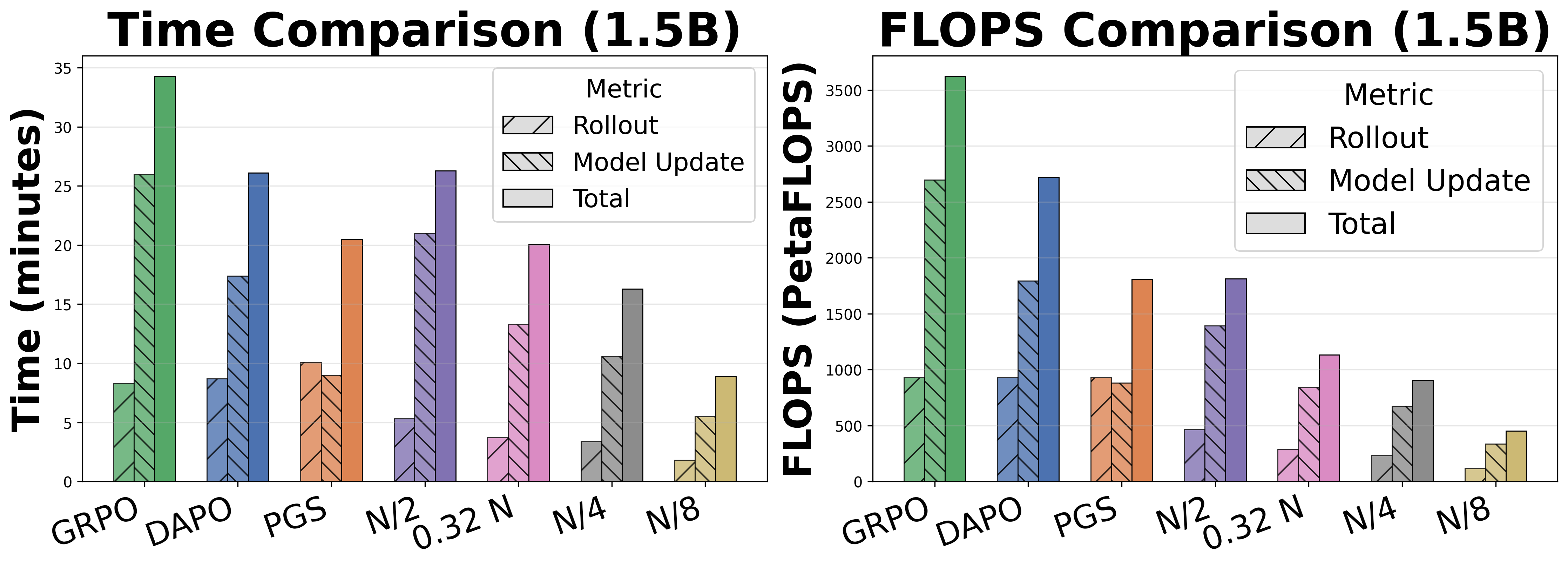}
    \caption{Training Time Cost and FLOPs across Baselines.}
    \label{fig:budget_reduced}
    \vspace{-.6em}
\end{figure}

Table~\ref{tab:budget_baselines_main} summarizes their efficiency and performance. Figure~\ref{fig:budget_reduced} reports the detailed time and FLOPs (rollout, model update, total).

Notably, AERO achieves substantially better cost-performance trade-offs across all model scales. On the 1.5B Math model, AERO requires less training time than the half-budget ($N/2$) setup and a similar amount of time as the $0.32N$ configuration, yet delivers significantly higher performance (Avg@8: 0.345 vs. 0.318).


\subsubsection{Performance on Code Generalization Tasks.}

Table~\ref{tab:code_efficiency_combined} presents the result.  AERO achieves the highest Avg@8 score (0.264) while requiring the least training time (6.95 min) and fewest FLOPs (2295 PetaFLOP), demonstrating that the efficiency gains observed on math reasoning transfer to code generation tasks.

\begin{table}[t]
\centering
\caption{Performance and efficiency comparison on code benchmarks for Qwen2.5-7B-Instruct-1M.}
\label{tab:code_efficiency_combined}
\begin{icmltable}
\adjustbox{max width=\columnwidth}{%
\begin{tabular}{@{}lcccccc c@{}}
\toprule
& \multicolumn{2}{c}{\textbf{Efficiency}} & \multicolumn{4}{c}{\textbf{Avg@8}} & \\
\cmidrule(lr){2-3}\cmidrule(lr){4-7}
\textbf{Method} & \textbf{Time (min)} & \textbf{FLOPs (PetaFLOP)} & \textbf{TACO} & \textbf{CodeContests} & \textbf{APPS} & \textbf{Total} & \textbf{$p$-value} \\
\midrule
AERO       & \cellcolor{table-blue!66} \textbf{6.95} & \cellcolor{table-blue!66} 2295
           & \cellcolor{table-blue!66} \textbf{0.179} & \cellcolor{table-blue!66} \textbf{0.321} & \cellcolor{table-blue!66} \textbf{0.337}
           & \cellcolor{table-blue!66} \textbf{0.264} & --- \\
GRPO       & 11.66 & 4488 & 0.157 & 0.296 & 0.298 & 0.238 & $<0.005$ \\
DAPO       & 7.75  & 2941 & 0.162 & 0.313 & 0.310 & 0.252 & $=0.034$ \\
Base Model & ---   & ---  & 0.094 & 0.195 & 0.188 & 0.149 & $<1\text{e-5}$ \\
\bottomrule
\end{tabular}
}  
\end{icmltable}
\end{table}


\subsection{Strategy Contribution Ablation}
\label{sec:strategy_ablation}

To quantify each strategy’s contribution to AERO’s efficiency gains, we measure per-step time savings from each component across all three model configurations.

\begin{figure}[H]
    \centering
        \vspace{-0.6em}
    \includegraphics[width=0.9\linewidth]{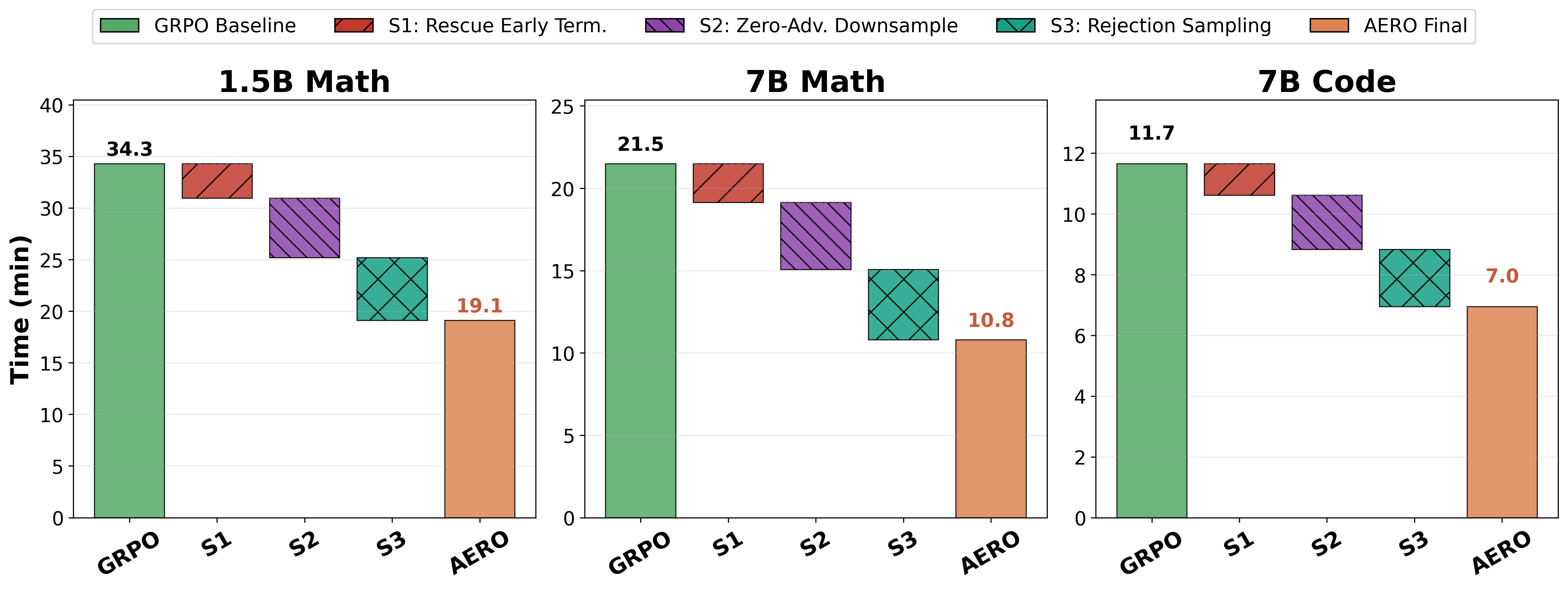}
    \caption{Waterfall decomposition of training time savings from GRPO to AERO across three models. Each bar shows the incremental reduction from: (S1) early termination on rescued queries, (S2) downsampling zero-advantage queries, and (S3) rejection sampling on low-success partial queries.}
    \label{fig:strategy_contribution}
\vspace{-.6em}
\end{figure}

\subsection{Ablation: Rejection Sampling}
\vspace{-5pt}
\begin{table}[H]
\centering
\caption{Ablation study: rejection sampling, advantage filtering, and Bayesian stabilization (Qwen2.5-Math-1.5B).}
\label{tab:ablation}
\begin{icmltable}
\adjustbox{max width=\columnwidth}{%
\begin{tabular}{@{}lcccc@{}}
\toprule
\textbf{Method} & \textbf{Time (min)} & \textbf{FLOPs (PF)} & \textbf{Avg@8} & \textbf{Pass@8} \\
\midrule
\textbf{AERO} & \cellcolor{table-blue}19.1 & \cellcolor{table-blue}1811 & \cellcolor{table-blue}0.345 & \cellcolor{table-blue}\textbf{0.612} \\
w/o Rejection Sampling & 22.9 {\scriptsize($\uparrow$19\%)} & 2099 {\scriptsize($\uparrow$16\%)} & 0.340 {\scriptsize($\downarrow$1.5\%)} & 0.606 {\scriptsize($\downarrow$1\%)} \\
w/o Advantage Filtering & 26.8 {\scriptsize($\uparrow$40\%)} & 2732 {\scriptsize($\uparrow$51\%)} & 0.329 {\scriptsize($\downarrow$5\%)} & 0.598 {\scriptsize($\downarrow$2\%)} \\
\midrule
\multicolumn{5}{l}{\textit{Bayesian Stabilization (batch=128)}} \\
\midrule
GRPO + Bayesian Prior  & 20.8 & 2186 & \cellcolor{table-blue}0.322 & 0.594 \\
GRPO & 21.9 & 2319 & 0.318 & 0.580 \\
\bottomrule
\end{tabular}%
}  
\end{icmltable}
\end{table}
\vspace{-5pt}

Here we evaluate the effect of disabling rejection sampling. 
Table~\ref{tab:ablation} (row 2) shows disabling rejection sampling leads to a massive \textbf{19\% increase in training time} and a \textbf{16\% increase in FLOPs}. 

To justify our choice of the downsampling ratio $k=1$, we examined the impact of varying $k$ from 1 to 4.  The computational cost, especially training time, grows monotonically with $k$ (see \S~\ref{app:downsampling_cost}).  For instance, moving from $k=1$ to $k=2$ increases training time by 7\%, to $k=3$ by 13\%, and to $k=4$ by 16\%.

\subsection{Non-Zero-Advantage Filtering.}
 
We also compare against simply filtering out zero-advantage rollouts while keeping the full budget $N=16$. As shown in Table~\ref{tab:ablation} (row 3), this naive filtering approach is less efficient and achieves lower performance than AERO.
We have see that the AERO gains higher efficiency while maintain a higher Avg@8 and Pass@8 score than purely filtering out the zero-advantage samples.


\subsection{Bayesian Stabilization: Gradient Norm Dynamics}

To highlight the impact of Bayesian estimation, we co
mpare the gradient norm behavior of the vanilla GRPO algorithm with a variant that incorporates a Bayesian prior for zero accuracy and all‑correct queries (denoted GRPO +Bayesian Prior). Both methods use a batch size of 128.

As shown in Table~\ref{tab:ablation} (rows 4--5), the Bayesian variant achieves higher Avg@8 and Pass@8 scores. 

The gradient norm analysis (Table~\ref{tab:bayesian_grad_norm} in \S~\ref{app:bayesian_grad}) shows that Bayesian estimation produces $\sim$1.5$\times$ larger gradient norms ($p < 0.001$).
Also, see \S~\ref{app:bayesian_grad} for full training dynamics.

\subsection{Effect of Iterative Rescue}

\begin{wrapfigure}{r}{0.5\columnwidth}
    \vspace{-1em}
    \centering
    \includegraphics[width=0.9\linewidth]{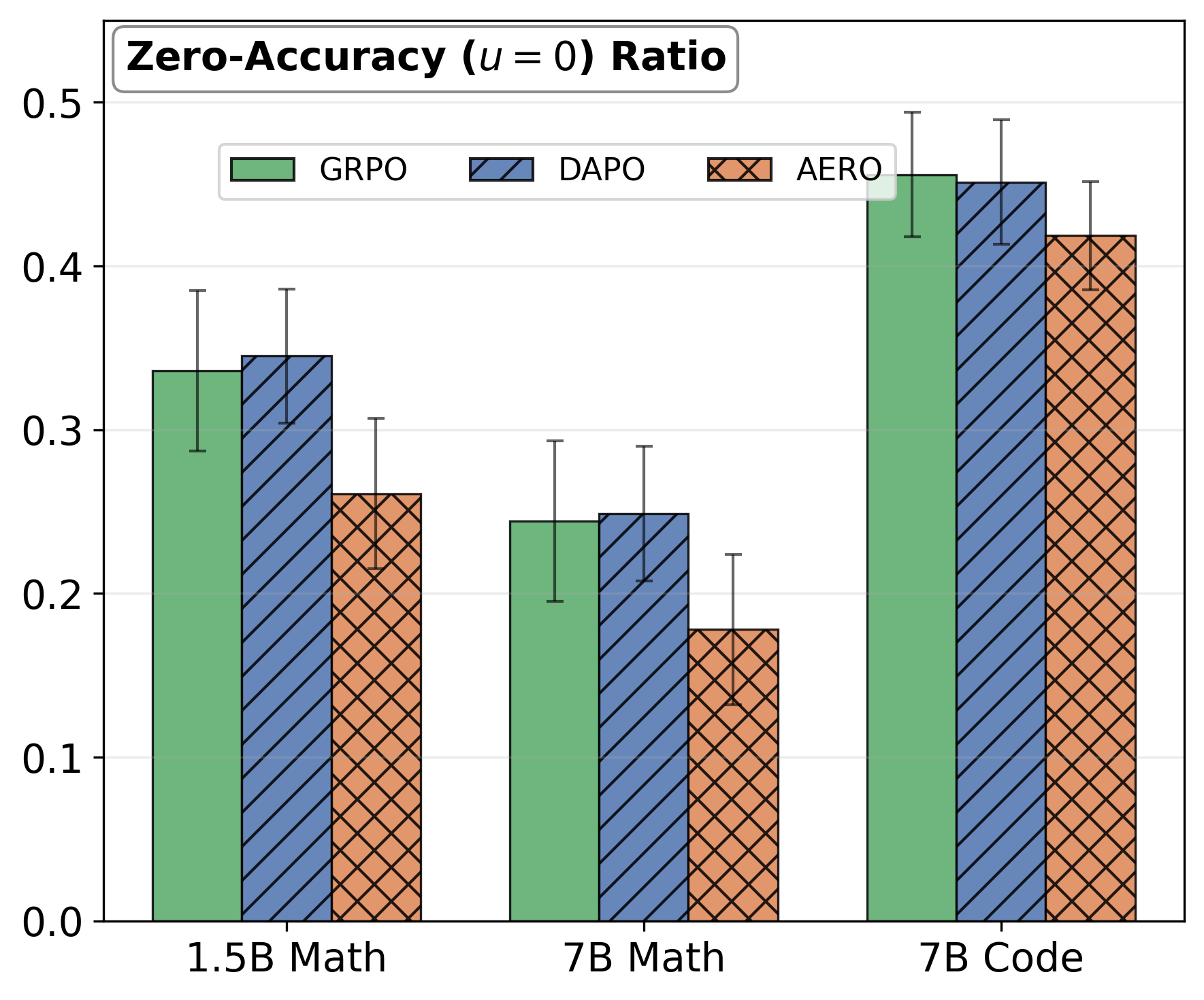}
    \caption{Average of Zero-Accuracy Ratio Across Training Steps}
    \label{fig:zero}
    \vspace{-1.5em}
\end{wrapfigure}
\vspace{-1pt}
Figure~\ref{fig:zero} compares the proportion of zero-accuracy problems across methods. AERO produces approximately $10\%$ fewer Zero-accuracy problems than GRPO and DAPO. 
Figure~\ref{fig:adv_overall} shows absolute advantage values across methods.  Larger advantage magnitudes enable stronger gradient signals for policy updates. AERO achieves approximately 30\% higher average advantage magnitude while also reducing zero-advantage samples, which would otherwise contribute no learning signal.

{\captionsetup[sub]{font=tiny}
\begin{figure}[ht]   
\centering

\begin{subfigure}{0.15\textwidth}
    \centering
    \includegraphics[width=\linewidth]{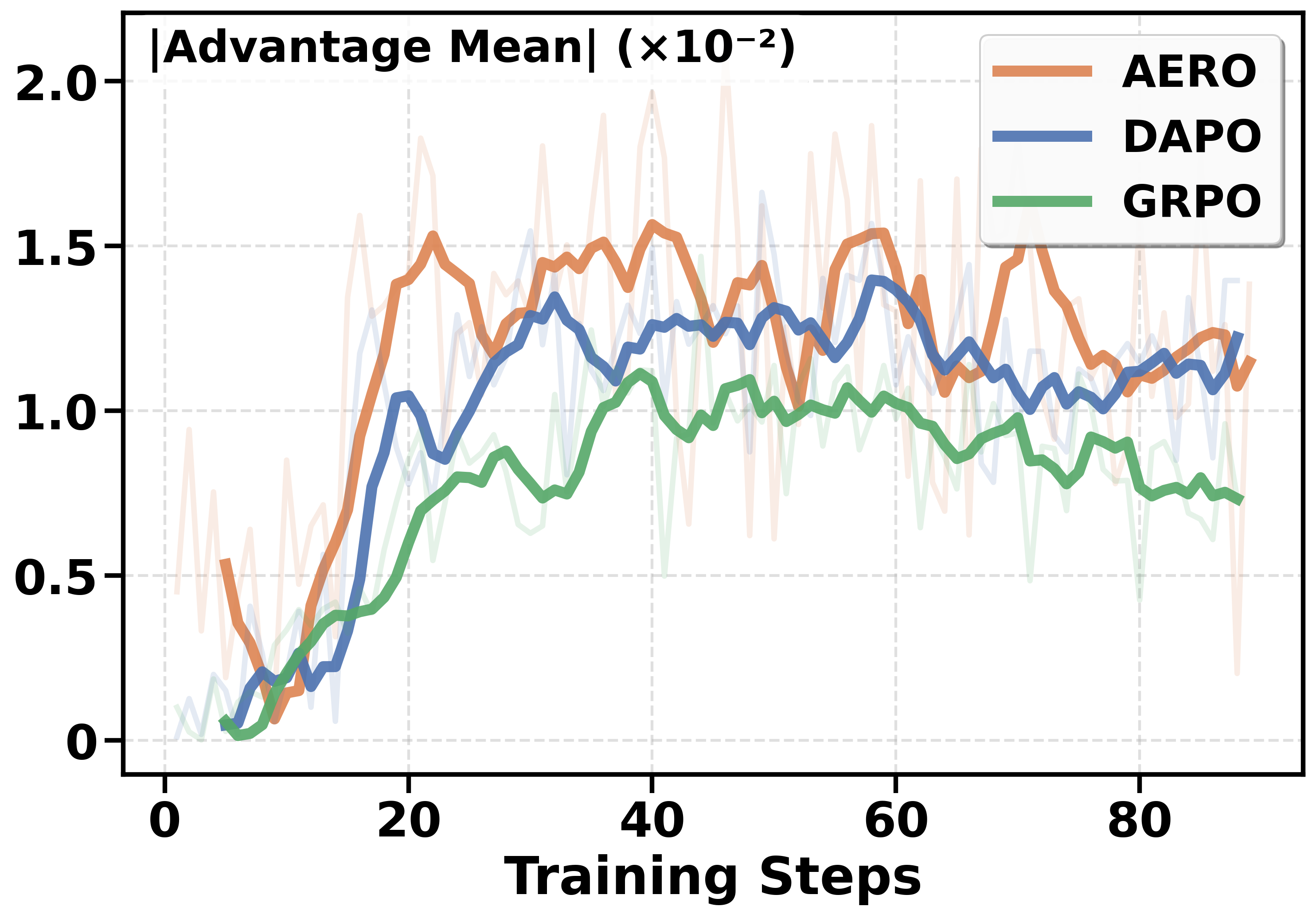}
    \caption{qwen2.5-math-1.5b}
    \label{fig:adv1}
\end{subfigure}\hfill
\begin{subfigure}{0.15\textwidth}
    \centering
    \includegraphics[width=\linewidth]{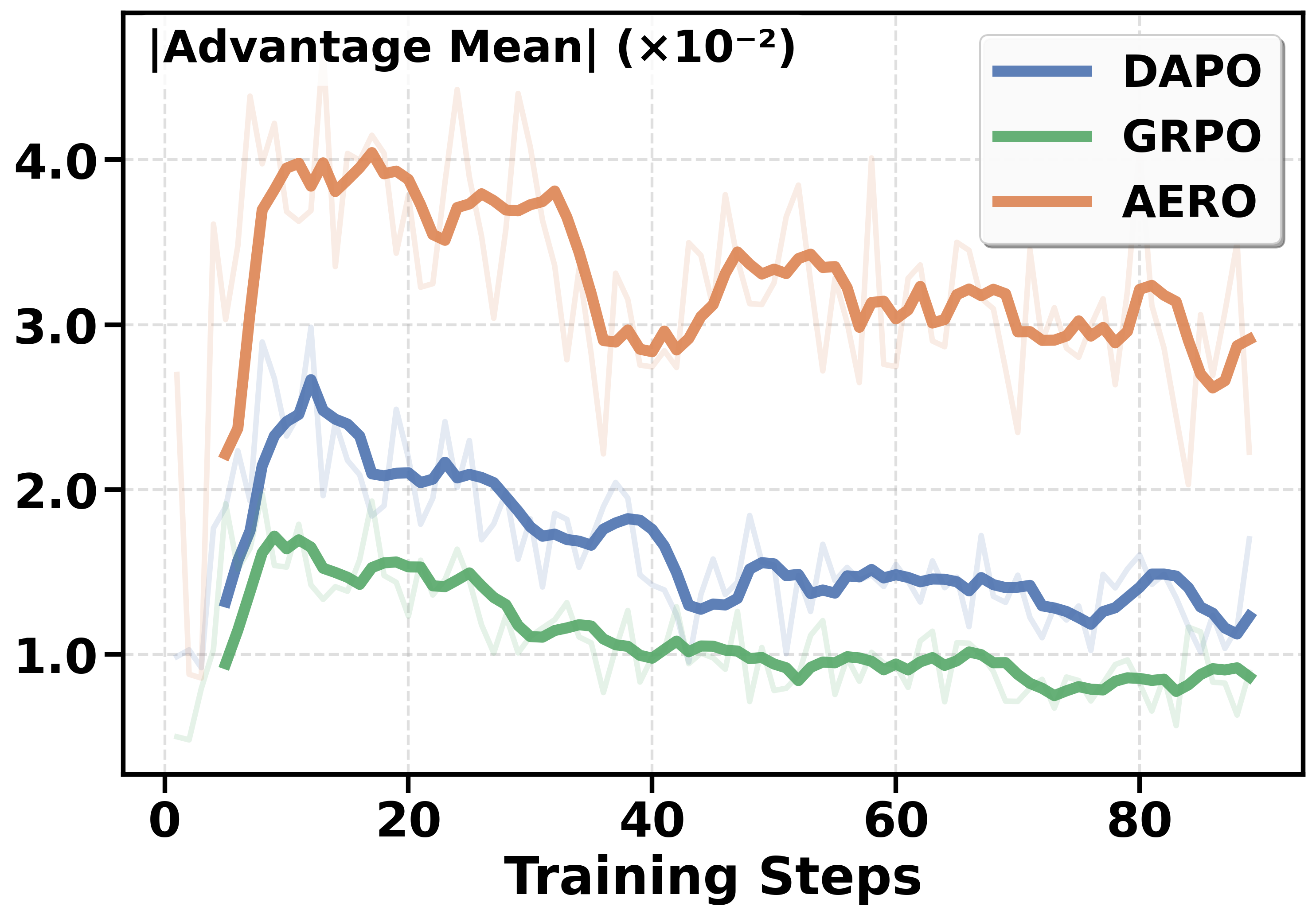}
    \caption{qwen2.5-math-7b}
    \label{fig:adv2}
\end{subfigure}\hfill
\begin{subfigure}{0.15\textwidth}
    \centering
    \includegraphics[width=\linewidth]{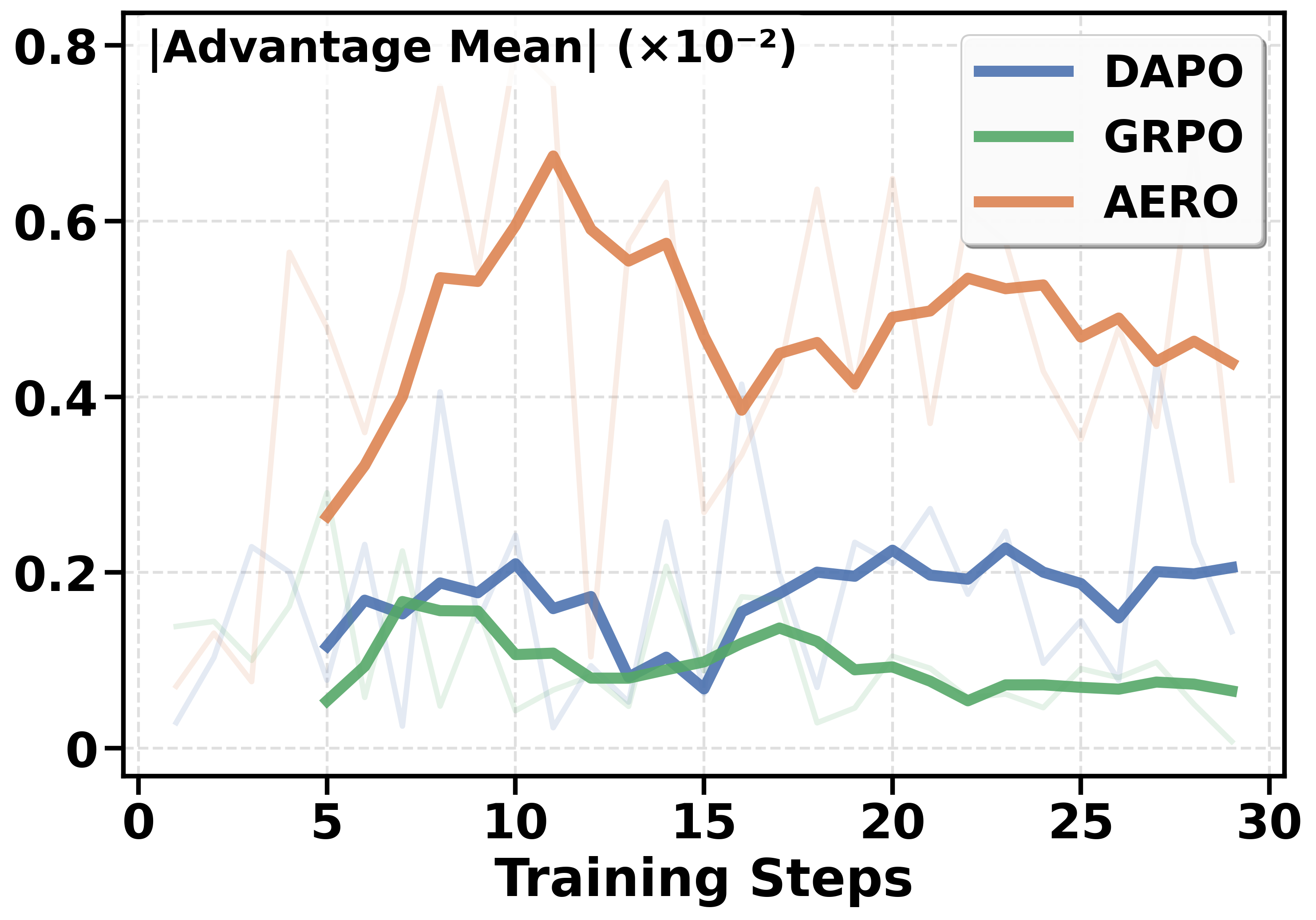}
    \caption{qwen2.5-7b-Instruct-1m}
    \label{fig:adv3}
\end{subfigure}

\caption{Absolute Advantage Value Comparison.}
\label{fig:adv_overall}
\end{figure}}

Additional analysis on whether it is beneficial to also down-sample high-success ($u \ge 0.5$) queries is provided in \S~\ref{sec:ablation_high_success_appendix}: the results show that discarding correct rollouts from such queries hurts accuracy with negligible time savings, justifying our decision to retain them.


\section{Conclusion}
We introduced AERO (Adaptive Efficient Rollout Optimization), which addresses key inefficiencies in GRPO through adaptive rollout reallocation, theoretically-grounded down-sampling, and Bayesian estimation to avoid gradient dead zones. On Qwen2.5 Math (1.5B, 7B), AERO reduces FLOPs by $\sim$2$\times$ and with$\sim$1.9$\times$ faster training  while improving Pass@8 and Avg@8. These gains transfer to code: on Qwen2.5-7B-Instruct, AERO achieves the best Avg@8 (0.264) with $\sim$2$\times$ fewer FLOPs and $\sim$1.7$\times$ faster training. AERO enables training less while learning more.

\section{Impact Statement}
This paper presents work whose goal is to advance the field of machine learning. There are many potential societal consequences of our work, none of which we feel must be specifically highlighted here.

\bibliographystyle{icml2026}
\bibliography{references}

\clearpage
\appendix
\onecolumn

\section{Theorem Supplementary}
\label{theorem_supp}
\subsection{Derivation of Policy-gradient form.}
\label{sup_proof_policy_gradient}

\begin{proof}
We consider the GRPO objective
\begin{equation}
J(\theta)
= \mathbb{E}_{q,\{o_i\}}\!\left[
\frac{1}{N}\sum_{i=1}^{N} \hat{A}_i \,\frac{1}{|o_i|}
\sum_{t=1}^{|o_i|}
\min\!\big(\rho_{i,t},\,\mathrm{clip}(\rho_{i,t},\,1-\epsilon,\,1+\epsilon)\big)
\;-\;\beta\, D_{\mathrm{KL}}(\pi_\theta \Vert \pi_{\mathrm{ref}})
\right],
\label{eq:grpo-obj}
\end{equation}
where the per-token importance ratio is
\(
\rho_{i,t}
= \dfrac{\pi_\theta(o_{i,t}\mid q, o_{i,<t})}{\pi_{\mathrm{old}}(o_{i,t}\mid q, o_{i,<t})}.
\)
Here \(\pi_{\mathrm{old}}\) and \(\pi_{\mathrm{ref}}\) are fixed (no gradients).

Differentiating \eqref{eq:grpo-obj} and using
$\nabla_\theta \rho_{i,t}=\rho_{i,t}\,\nabla_\theta \log\pi_\theta(o_{i,t}\!\mid q,o_{i,<t})$,
we obtain the exact gradient (for the one–sided clipping used in GRPO):
\begin{equation}
\label{eq:full-grad-with-ind}
\nabla_\theta J(\theta)
= \mathbb{E}_{q,\{o_i\}}\!\left[
\frac{1}{N}\sum_{i=1}^{N}\frac{\hat A_i}{|o_i|}
\sum_{t=1}^{|o_i|}
\mathbb{I}\!\big[\rho_{i,t}\le 1+\epsilon\big]\,
\rho_{i,t}\,\nabla_\theta \log\pi_\theta(o_{i,t}\!\mid q,o_{i,<t})
\right]
\;-\;\beta\,\nabla_\theta D_{\mathrm{KL}}(\pi_\theta\Vert \pi_{\mathrm{ref}}).
\end{equation}

In the common small-step (trust-region) regime where the policy ratios remain within the clipping range,
$\rho_{i,t}\le 1+\epsilon$ for all $t$,\footnote{This is satisfied exactly when $\epsilon=0$, or approximately when updates are sufficiently small so that ratios stay inside the trust region.}
the indicator in \eqref{eq:full-grad-with-ind} is identically one and the gradient reduces to
\begin{equation}
\label{eq:simplified-grad2}
\nabla_\theta J(\theta)
= \mathbb{E}_{q,\{o_i\}}\!\left[
\frac{1}{N}\sum_{i=1}^{N}\frac{\hat A_i}{|o_i|}
\sum_{t=1}^{|o_i|}
\rho_{i,t}\,\nabla_\theta \log\pi_\theta(o_{i,t}\!\mid q,o_{i,<t})
\right]
\;-\;\beta\,\nabla_\theta D_{\mathrm{KL}}(\pi_\theta\Vert \pi_{\mathrm{ref}}).
\end{equation}

Since \(\pi_{\mathrm{old}}\) is constant w.r.t.~\(\theta\),
\beq 
\nabla_\theta \rho_{i,t} & 
= \frac{\nabla_\theta \pi_\theta(o_{i,t}\mid q, o_{i,<t})}{\pi_{\mathrm{old}}(o_{i,t}\mid q, o_{i,<t})}
= \frac{\pi_\theta(o_{i,t}\mid q, o_{i,<t} )}{\pi_{\mathrm{old}}(o_{i,t}\mid q,o_{i,<t})} \nabla_\theta \log \pi_\theta(o_{i,t}\mid q, o_{i,<t}) \\ &
= \rho_{i,t}\, \nabla_\theta \log \pi_\theta(o_{i,t}\mid q,o_{i,<t}).
\label{eq:ratio-grad}
\beq

Define \(f_\epsilon(\rho)=\min\!\big(\rho,\,\mathrm{clip}(\rho,1-\epsilon,1+\epsilon)\big)\).

Thus, we have 
\bea 
\min\!\big(\rho_{i,t},\,\mathrm{clip}(\rho_{i,t},\,1-\epsilon,\,1+\epsilon)\big)=:f_\epsilon(\rho_{i,t}).
\eea 

A simple case analysis yields
\[
f_\epsilon(\rho) =
\begin{cases}
\rho, & \rho \le 1+\epsilon,\\
1+\epsilon, & \rho > 1+\epsilon,
\end{cases}
\qquad\Rightarrow\qquad
\frac{\partial f_\epsilon}{\partial \rho} =
\begin{cases}
1, & \rho \le 1+\epsilon,\\
0, & \rho > 1+\epsilon,
\end{cases}
\]
(except at \(\rho=1+\epsilon\) where a subgradient can be taken).
Thus,
\begin{equation}
\nabla_\theta f_\epsilon(\rho_{i,t})
= \mathbb{I}\!\left[\rho_{i,t}\le 1+\epsilon\right]\,
\nabla_\theta \rho_{i,t}
= \mathbb{I}\!\left[\rho_{i,t}\le 1+\epsilon\right]\,
\rho_{i,t}\,\nabla_\theta \log \pi_\theta(o_{i,t}\mid \cdot),
\label{eq:clipped-grad}
\end{equation}
where we used \eqref{eq:ratio-grad}.

Using linearity of expectation and taking gradients inside,
\begin{align}
\nabla_\theta J(\theta)
&= \mathbb{E}_{q,\{o_i\}}\!\left[
\frac{1}{N}\sum_{i=1}^{N} \hat{A}_i \,\frac{1}{|o_i|}
\sum_{t=1}^{|o_i|}
\nabla_\theta f_\epsilon(\rho_{i,t})
\right]
\;-\;\beta\, \nabla_\theta D_{\mathrm{KL}}(\pi_\theta \Vert \pi_{\mathrm{ref}}) \nonumber\\
&= \mathbb{E}_{q,\{o_i\}}\!\left[
\frac{1}{N}\sum_{i=1}^{N} \frac{\hat{A}_i}{|o_i|}
\sum_{t=1}^{|o_i|}
\mathbb{I}\!\left[\rho_{i,t}\le 1+\epsilon\right]\,
\rho_{i,t}\,\nabla_\theta \log \pi_\theta(o_{i,t}\mid q,o_{i,<t})
\right]
\;-\;\beta\, \nabla_\theta D_{\mathrm{KL}}(\pi_\theta \Vert \pi_{\mathrm{ref}}).
\label{eq:full-grad}
\end{align}

\textbf{Unclipped (or small-step) limit.}
If we assume  \(\rho_{i,t}\le 1+\epsilon\) holds for all $t$, then
\(\mathbb{I}[\rho_{i,t}\le 1+\epsilon]=1\),
\(f_\epsilon(\rho_{i,t})=\rho_{i,t}\),
and \eqref{eq:full-grad} simplifies to
\begin{equation}
\nabla_\theta J(\theta)
= \mathbb{E}_{q,\{o_i\}}\!\left[
\frac{1}{N}\sum_{i=1}^{N} \frac{\hat{A}_i}{|o_i|}
\sum_{t=1}^{|o_i|}
\rho_{i,t}\,\nabla_\theta \log \pi_\theta(o_{i,t}\mid q,o_{i,<t})
\right]
\;-\;\beta\, \nabla_\theta D_{\mathrm{KL}}(\pi_\theta \Vert \pi_{\mathrm{ref}}),
\label{eq:simple-grad}
\end{equation}

When some tokens are upper-clipped, define the un-clipped index set
$U_i=\{t: \rho_{i,t}\le 1+\epsilon\}$ and the clipped set $C_i=\{t:\rho_{i,t}>1+\epsilon\}$.
By the piecewise derivative $\partial f_\epsilon/\partial\rho=1$ for $\rho\le 1+\epsilon$ and $0$ for $\rho>1+\epsilon$, we have
\begin{equation}
\nabla_\theta J(\theta)
= \mathbb{E}_{q,\{o_i\}}\!\left[
\frac{1}{N}\sum_{i=1}^{N} \frac{\hat A_i}{|o_i|}
\sum_{t\in U_i} \rho_{i,t}\,\nabla_\theta \log \pi_\theta(o_{i,t}\mid q,o_{i,<t})
\right]
-\beta\,\nabla_\theta D_{\mathrm{KL}}(\pi_\theta\Vert\pi_{\mathrm{ref}}).
\label{eq:clipped-grad-final}
\end{equation}
In other words, only tokens whose ratios remain within the clip window contribute to the policy-gradient term; upper-clipped tokens contribute zero to that term, but the KL regularizer still applies.
\end{proof}


\subsection{Proof of Theorem~\ref{thm:grad_norm_maximization}} 
\vspace{5pt}
\label{sup_proof_thm_grad_norm_maximization}

\begin{theorem}[GRPO Gradient Norm]
Let a training subset of $n = c+m$ rollouts consists of $c$ correct rollouts ($r_i=1$) from a set $F_C$ and $m$ incorrect rollouts ($r_i=0$) from a set $F_S$.

Consider the group relative policy objective function, under the simplifying assumptions that all rollouts have a uniform length $L$, and the KL penalty $\beta=0$, based on \eqref{eq:simplified-grad}, let $G(\theta)$ be the group gradient under this, which is
\begin{equation}
\label{label:gradient_for_theorem}
G(\theta)
= \mathbb{E}_{q,\{o_i\}}\!\left[
\frac{1}{NL}\sum_{i=1}^{N}
\sum_{t=1}^{L} \hat A_i
\rho_{i,t}\,\nabla_\theta \log\pi_\theta(o_{i,t}\!\mid q,o_{i,<t})
\right]
\end{equation}
where $\rho_{i,t} = \frac{\pi_\theta(o_{i,t} | q, o_{i,<t})}{\pi_{old}(o_{i,t} | q, o_{i,<t})}$. We have its squared norm,
\[
\|\hat{G}\|^2 \le \frac{c \cdot m}{(c+m)^2} \cdot \mathbb{E}\|\bar{v}_C - \bar{v}_F\|^2,
\]
where $\frac{c \cdot m}{(c+m)^2}$ 
is a scalar factor dependent on the group set composition. The vectors $\bar{v}_C = \frac{1}{c}\sum_{j \in F_C} v_j$ and $\bar{v}_F = \frac{1}{m}\sum_{j \in F_S} v_j$, with $v_j = \frac{1}{L}\sum_{t=1}^{L} \nabla_\theta \log\pi_\theta(o_{j,t}| q, o_{i,<t})$,and $\|\bar{v}_C - \bar{v}_F\|^2$ represents the squared distance between the mean policy gradient directions of correct and incorrect rollouts.
\end{theorem}

\begin{proof}

By Equation~\ref{eq:simplified-grad2},  the gradient of this objective with respect to the policy parameters $\theta$ is, 
\begin{equation}
\label{eq:simplified-grad2}
\nabla_\theta J(\theta)
= \mathbb{E}_{q,\{o_i\}}\!\left[
\frac{1}{N}\sum_{i=1}^{N} \frac{\hat A_i}{|o_i|}
\sum_{t\in U_i} \rho_{i,t}\,\nabla_\theta \log \pi_\theta(o_{i,t}\mid q,o_{i,<t})
\right]
-\beta\,\nabla_\theta D_{\mathrm{KL}}(\pi_\theta\Vert\pi_{\mathrm{ref}}).
\end{equation}

Under simplified assumption, we have its simplified form as

\begin{equation}
\nabla_\theta J(\theta)
= \mathbb{E}_{q,\{o_i\}}\!\left[
\frac{\hat A_i}{NL}\sum_{i=1}^{N}\sum_{t=1}^L  \rho_{i,t}\,\nabla_\theta \log \pi_\theta(o_{i,t}\mid q,o_{i,<t})
\right].
\end{equation}

Let the per-trajectory policy gradient direction vector be $v_i = \nabla_\theta \left( \sum_{t=1}^{L} \log \pi_\theta(o_{i,t}\mid q,o_{i,<t}) \right)$. The random group gradient for a specific query is defined as:
$$
\hat{G}(\theta) = \frac{1}{N} \sum_{i=1}^{N} \hat{A}_i v_i
$$
where $\hat{A}_i$ is the standardized advantage. For Bernoulli rewards, $\hat{A}_i = a_i = (r_i - \mu)/\sigma$. Thus, The normalized group gradient over the selected subset is defined as:
$$
\hat{G}(\theta) = \frac{1}{N} \sum_{j=1}^{N} a_j v_j = \frac{1}{c+m} \left( \sum_{i \in F_C} a_+ v_i + \sum_{i \in F_S} a_- v_i \right).
$$


where $a_j = (r_j - \mu)/\sigma$ is the standardized advantage, with $\mu = c/(c+m)$ and $\sigma^2 = \mu(1-\mu)$. Specifically, $a_+ = (1-\mu)/\sigma$ for correct samples and $a_- = -\mu/\sigma$ for incorrect samples.

We approximate the group gradient by using the average direction vectors $\bar{v}_C = \frac{1}{c}\sum_{j \in F_C} v_j$ and $\bar{v}_F = \frac{1}{m}\sum_{j \in F_S} v_j$:
$$
\hat{G}(\theta) = \frac{1}{c+m} \left( c \cdot a_+ \cdot \bar{v}_C + m \cdot a_- \cdot \bar{v}_F \right)
$$
Substituting the expressions for $a_+, a_-, \mu$, and $1-\mu$:
\begin{align*}
\hat{G}(\theta) &= \frac{1}{c+m} \left( c \cdot \frac{m/(c+m)}{\sigma} \bar{v}_C - m \cdot \frac{c/(c+m)}{\sigma} \bar{v}_F \right) \\
&= \frac{cm}{\sigma(c+m)^2} (\bar{v}_C - \bar{v}_F)
\end{align*}

The squared norm of expecation of this approximate gradient is:
$$
\|G(\theta)\|^2 = \left( \frac{cm}{\sigma(c+m)^2} \right)^2 \|\mathbb E (\bar{v}_C - \bar{v}_F)\|^2
$$

Now, we substitute the expression for the variance, $\sigma^2 = \mu(1-\mu) = \frac{cm}{(c+m)^2}$:
\begin{align*}
\|G(\theta)\|^2 &= \frac{(cm)^2}{\sigma^2(c+m)^4} \|\mathbb E (\bar{v}_C - \bar{v}_F)\|^2 \\
&= \frac{(cm)^2}{\left(\frac{cm}{(c+m)^2}\right)(c+m)^4} \|\mathbb E(\bar{v}_C - \bar{v}_F)\|^2 \\
&\leq \frac{cm}{(c+m)^2} \mathbb E\|\bar{v}_C - \bar{v}_F\|^2,
\end{align*}
where the last inequality is based on Cauchy Schwarz inequality.


\end{proof}

\subsection{Proof of Lemma~\ref{lma:grad_norm_maximization}}
\label{sup_proof_lemma_grad_norm_maximization}
\vspace{5pt}
\begin{lemma}[Maximization of the GROP Gradient Norm]
For the same set up in Theorem \ref{thm:grad_norm_maximization}, 
\[
\|\hat{G}\|^2 \le z \cdot \mathbb{E}\|\bar{v}_C - \bar{v}_F\|^2,
\] the balance term $z$ is maximized if and only if $m=c$.
\end{lemma}

\begin{proof}
To maximize this expression with respect to $m$ (for a fixed $c$), we focus on the balance term. Let $f(m) = \frac{cm}{(c+m)^2}$. Taking the derivative with respect to $m$ and setting it to zero yields $m=c$. Therefore, the balance term is maximized if and only if $m=c$, which corresponds to setting $k=1$. Under the assumption that the directional term is approximately constant for small changes in the composition of $F_S$, this choice maximizes the overall gradient norm.
\end{proof}

\section{Experiment Supplementary}



\subsection{Datasets}

Our data set spanning on mathematical reasoning and code questions. Below is the data set we used.

\begin{table}[h!]
    \centering
    \caption{Overview of Training and Evaluation Datasets}
    \label{tab:datasets_appendix}
    \begin{tabular}{@{}p{3.5cm} p{2.5cm} p{8cm}@{}}
        \toprule
        \textbf{Dataset} & \textbf{Type} & \textbf{Description} \\
        \midrule
        \multicolumn{3}{l}{\textbf{Training Dataset (Math Reasoning)}} \\
        \cmidrule(r){1-1}
        OpenR1-AERO & Training & A 40k curated math problems from the OpenR1-Math-220k corpus, used by  \cite{zhao20251}. \\
        \addlinespace[1em] 
        \multicolumn{3}{l}{\textbf{Evaluation Datasets}} \\
        \cmidrule(r){1-1}
        AIME (2024) & Evaluation & American Invitational Mathematics Examination; competition-level high school math. \\
        \addlinespace
        AMC (2023) & Evaluation & American Mathematics Competitions; qualifier benchmark for AIME. \\
        \addlinespace
        MATH-500 & Evaluation & 500-problem subset of the MATH benchmark, spanning pre-algebra to calculus \citep{lightman2023let}. \\
        \addlinespace
        MinervaMath & Evaluation & Quantitative reasoning problems across STEM fields \citep{lewkowycz2022solving}. \\
        \addlinespace
        OlympiadBench & Evaluation & National and international Olympiad problems requiring advanced reasoning \citep{he2024olympiadbench}. \\
        \midrule
        \multicolumn{3}{l}{\textbf{Training Dataset (Code Generation)}} \\
        \cmidrule(r){1-1}
        Prime-1.4K & Training & 1.4k curated code data from the original Eurus-45k, used by \cite{cui2025process}. \\
        \addlinespace[1em] 
        \multicolumn{3}{l}{\textbf{Evaluation Datasets}} \\
        \cmidrule(r){1-1}
        APPS  & Evaluation & \cite{hendrycks2021measuring} \\
          CodeContests & Evaluation & \cite{li2022competition}  \\
        \addlinespace
        TACO & Evaluation & \cite{li2023taco}. \\
        \bottomrule
    \end{tabular}
\end{table}

\newpage

\subsection{Baselines and Benchmarks}

\subsubsection{Evaluation Metrics}

We evaluate AERO using the following comprehensive metrics:

\textbf{Primary Performance Metrics:}
\begin{itemize}
    
    \item \textbf{Pass@$n$:} Traditional pass@$n$ metric for problems with exactly $n$ samples:
    \begin{equation}
    \text{Pass@}n = \frac{1}{|P_n|} \sum_{p \in P_n} \mathbb{I}\left[\max_{i=1}^n c_{p,i} = 1\right]
    \end{equation}
    where $P_n$ is the set of problems allocated exactly $n$ samples, and \(c_{p,i}\in\{0,1\}\) is correctness of the \(i\)-th sample for problem \(p\)..

    \item \textbf{Avg@n:} For the same set,
\begin{equation}
\text{Avg@}n \;=\; \frac{1}{|P_n|}\sum_{p\in P_n}\; \frac{1}{n}\sum_{i=1}^{n} c_{p,i}.
\end{equation}
    This measures expected performance under the learned allocation policy.


    \item \textbf{Computational Efficiency:} Total FLOPs per step including rollout generation and training:
    \begin{equation}
    \text{Total FLOPs} = \text{FLOPs}_{\text{rollout}} + \text{FLOPs}_{\text{training}}
    \end{equation}
    where $\text{FLOPs}_{\text{training}} = 6 \times N_{\text{params}} \times N_{\text{tokens}}$ for forward and backward passes, and $\text{FLOPs}_{\text{rollout}} = 2 \times N_{\text{params}} \times N_{\text{tokens}}$, 
\end{itemize}
Here, \( N_{\text{params}} \) denotes the total number of model parameters, and \( N_{\text{tokens}} \) represents the total number of tokens across all forward passes, including both input and generated tokens during rollout and training per step.

\subsection{Additional Results}

\subsubsection{Additional Budget Baseline Results}
\label{app:additional_budget_baselines}

Table~\ref{tab:budget_baselines_appendix} presents additional budget baseline comparisons for Qwen2.5-7B-Math and Qwen2.5-1.5B-Base models. Consistent with the main results, AERO achieves the best efficiency-performance trade-off across all model configurations. For the 7B model, AERO matches the training time of the reduced-budget baseline ($N=6$) while achieving significantly higher accuracy (Avg@8: 0.48 vs. 0.455). The advantage is even more pronounced on the Base model, where AERO delivers a 4$\times$ improvement in Avg@8 compared to the time-matched $N=5$ baseline (0.234 vs. 0.056), demonstrating that adaptive rollout allocation is particularly beneficial for weaker models that produce more zero-accuracy queries.

\begin{table}[!h]
\centering
\caption{Computational budget and alternative sampling strategies. 
The best-performing method is highlighted in blue, and time-matched baselines for comparison are in light blue.}
\vspace{2pt}
\label{tab:budget_baselines_appendix}
\begin{icmltable}
\adjustbox{max width=\columnwidth}{%
\begin{tabular}{@{}lcccc@{}}
\toprule
\textbf{Method} & \textbf{Time (min)} & \textbf{FLOPs (Total)} & \textbf{Avg@8} (P-value) & \textbf{Pass@8} \\
\midrule
\multicolumn{5}{l}{\textbf{Qwen2.5-1.5B-Math}} \\
\midrule
GRPO & 34.3 & 3627 & 0.3401 ($=0.28$) & 0.6090 \\
DAPO & 26.1 & 2722 & 0.3405 ($=0.31$) & 0.6100 \\
AERO  & \cellcolor{table-blue!66} \textbf{19.1} & \textbf{1811} & \cellcolor{table-blue!66} \textbf{0.3452} & \cellcolor{table-blue!66} \textbf{0.6121} \\
N/2 & 23.7 & 1863 & 0.3211 ($=0.004$) & 0.5954 \\
0.32N &\cellcolor{table-blue!30}   20.1 & 1133 & \cellcolor{table-blue!30}  0.3184 ($=0.002$) & \cellcolor{table-blue!30}  0.5891 \\
N/4 & 16.3 & 906 & 0.3133 ($=0.001$) & 0.5870 \\
N/8 & 8.9 & 453 & 0.2820 ($<1\text{e-4}$) & 0.5791 \\
Base & --- & --- & 0.1617 ($<1\text{e-6}$) & 0.5050 \\
\midrule
\multicolumn{5}{l}{\textbf{Qwen2.5-7B-Math}} \\
\midrule
AERO & \cellcolor{table-blue!50}10.8 & 9181 & \cellcolor{table-blue!66} 0.4800 & \cellcolor{table-blue!66} 0.7000 \\
DAPO & 15.8 & 13743 & 0.4750 ($=0.29$) & 0.6910 \\
GRPO & 21.5 & 17334 & 0.4690 ($=0.05$) & 0.6860 \\
N=6 & \cellcolor{table-blue!30}  10.6 & 6500 & \cellcolor{table-blue!30}  0.4550 ($=0.003$) & \cellcolor{table-blue!30}  0.6790 \\
Base & --- & --- & 0.3110($<1\text{e-6}$) & 0.6220 \\
\midrule
\multicolumn{5}{l}{\textbf{Qwen2.5-1.5B-Base}} \\
\midrule
AERO & \cellcolor{table-blue!66} 19 & 1849 & \cellcolor{table-blue!66} 0.2340 & \cellcolor{table-blue!66} 0.4990 \\
DAPO & 23.8 & 2289 & 0.2260 ($=0.12$) & 0.4770 \\
GRPO & 36.4 & 3714 & 0.2120 ($=0.006$) & 0.4460 \\
N=5 & \cellcolor{table-blue!30}  19.8 & 1160 & \cellcolor{table-blue!30}  0.0560 ($<1\text{e-6}$) & \cellcolor{table-blue!30}  0.2510 \\
Base & --- & --- & 0.0340($<1\text{e-6}$) & 0.1820 \\
\bottomrule
\end{tabular}%
}  \end{icmltable}
\vspace{-.6em}
\end{table}

\subsubsection{Downsampling Ratio Cost Analysis}
\label{app:downsampling_cost}


\begin{figure}[H]
\centering
\includegraphics[width=0.35\textwidth]{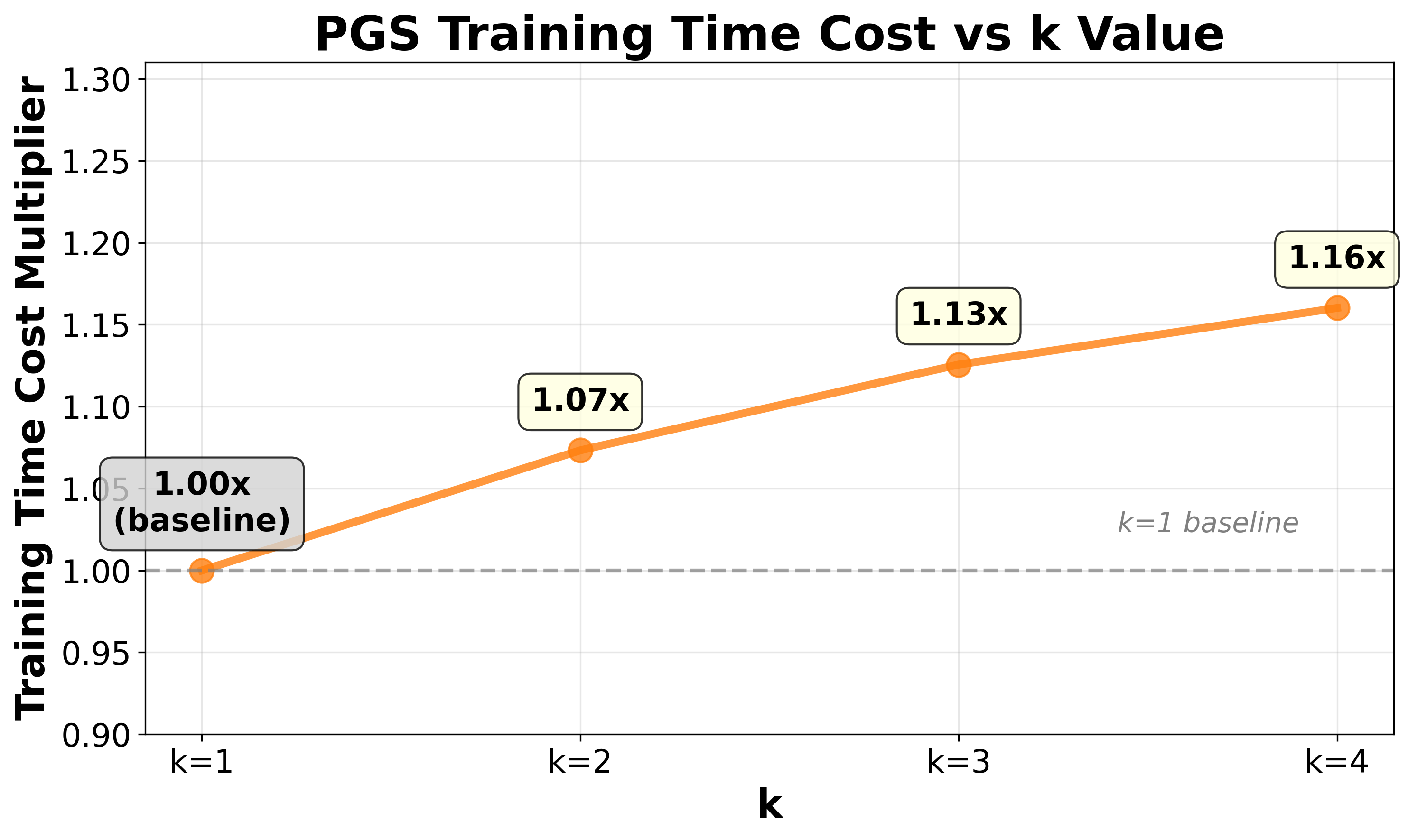}
\caption{Training Time Cost vs. $k$}
    \label{fig:cost_vs_k_appendix}
\end{figure}

Figure~\ref{fig:cost_vs_k_appendix} shows that the training‑time cost rises steadily as the downsampling ratio $k$ increases: a 7\% increase from $k=1$ to $k=2$, 13\% from $k=1$ to $k=3$, and 16\% from $k=1$ to $k=4$.


\subsubsection{Bayesian Stabilization: Full Training Dynamics}
\label{app:bayesian_grad}

Table~\ref{tab:bayesian_grad_norm} and Figure~\ref{fig:gradient_norm_vs_pg} visualize gradient dynamics during training. While both methods show decaying gradient norms as policies converge, the Bayesian-stabilized variant maintains $\sim$1.5$\times$ larger gradients ($p < 0.001$). This demonstrates that incorporating a Bayesian posterior for zero-accuracy and all-correct queries effectively mitigates the gradient dead-zone problem, yielding more stable and effective training dynamics.


\begin{figure}[H]
\centering
\includegraphics[width=0.35\textwidth]{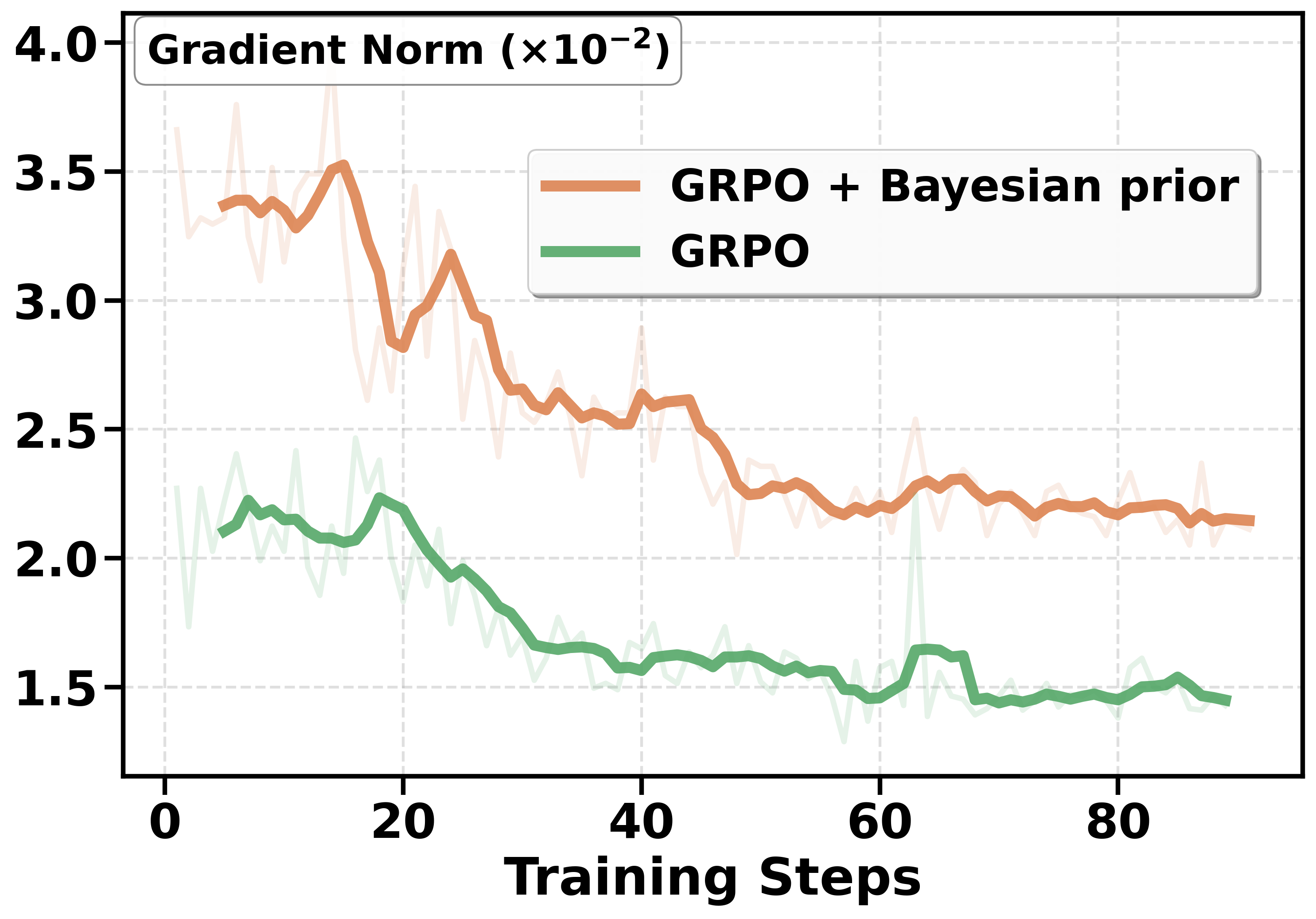}
\caption{Gradient norm: Bayesian vs GRPO.}
\label{fig:gradient_norm_vs_pg}
\end{figure}

\begin{table}[H]
\centering
\caption{Actor gradient norm ($\times 10^{-2}$) at steps 70/80/90 for GRPO + Bayesian Prior and GRPO. The Bayesian variant shows $\sim$1.5$\times$ larger gradient norms.}
\label{tab:bayesian_grad_norm}
\begin{icmltable}
\adjustbox{max width=1.2\columnwidth}{%
\begin{tabular}{@{}lccc c@{}}
\toprule
& \multicolumn{3}{c}{\textbf{Grad Norm ($\times 10^{-2}$)}} & \\
\cmidrule(lr){2-4}
\textbf{Method} & 70 & 80 & 90 & \textbf{$p$-value} \\
\midrule
\textbf{GRPO + Bayesian Prior} &
\cellcolor{table-blue}2.21 & \cellcolor{table-blue}2.22 & \cellcolor{table-blue}2.10 & — \\
\textbf{GRPO} & 1.46 & 1.38 & 1.40 & $<0.001$ \\
\bottomrule
\end{tabular}
}  
\end{icmltable}
\end{table}

\subsubsection{Effect of Discarding High-Success Rollouts}
\label{sec:ablation_high_success_appendix}
We conduct rejection sampling on high-success partial problems ($u \geq 0.5$). For example, consider a query with 8 rollouts, 5 correct and 3 incorrect. Under a $1:1$ scheme, we retain 3 correct and 3 incorrect rollouts. 
Compared to AERO (Table~\ref{tab:ablation_downsample}), 
training time is nearly identical (19.0 vs.\ 19.1 minutes per step) with similar rollout counts (1805 vs.\ 1811). 
However, accuracy drops noticeably, with Avg@8 falling from 0.3452 to 0.3321 and Pass@8 from 0.6121 to 0.6100 (Table~\ref{tab:ablation_downsample}). 
This shows that discarding correct rollouts from high-success queries degrades performance without yielding significant cost savings.

\begin{table}[H]
\centering
\caption{Efficiency and performance with rejection sampling on queries with $u\ge 0.5$.}
\label{tab:ablation_downsample}
\begin{icmltable}
\adjustbox{max width=1.6\columnwidth}{
\begin{tabular}{@{}lcccc@{}}
\toprule
\textbf{Method} & \textbf{Time (min)} & \textbf{FLOPs (PetaFLOP)} & \textbf{Avg@8} & \textbf{Pass@8} \\
\midrule
\multicolumn{5}{l}{\underline{\textbf{Qwen2.5-1.5B-Math}}} \\
GRPO & 34.3 & 3627 & 0.3401 & 0.6090 \\
DAPO & 26.1 & 2722 & 0.3405 & 0.6100 \\
AERO & 19.1 & 1811 & \cellcolor{table-blue!66} \textbf{0.3452} & \cellcolor{table-blue!66} \textbf{0.6121} \\
AERO (rejection sampling, $u\ge 0.5$) & 19.0 & 1805 & 0.3321 & 0.6100 \\
\bottomrule
\end{tabular}
}  
\end{icmltable}
\end{table}

\section{Learning Framework}
\label{sec:framework}

We implement AERO 
using the verl framework~\citep{sheng2025hybridflow}, a distributed RL system designed for LLMs. Our approach employs GRPO with specialized resource allocation mechanisms.

\paragraph{Distributed Architecture:}
The framework utilizes a Ray-based single controller design with multiple worker groups:
\begin{itemize}
    \item \textbf{ActorRollout Worker:} A hybrid worker that combines policy optimization (actor) and sequence generation (rollout) using FSDP (Fully Sharded Data Parallel) for efficient memory usage.
    \item \textbf{Critic Worker:} Computes value estimates $V(s_t)$ using a separate value network with FSDP sharding
    \item \textbf{Reference Policy Worker:} Maintains the reference policy $\pi_{\text{ref}}$ for KL regularization (optional). 
\end{itemize}

\paragraph{Model Architecture (Math \& Code):}
We train and evaluate AERO on multiple Qwen2.5-family backbones:
\begin{itemize}
    \item \textbf{Math models:} Qwen2.5-Math-1.5B and Qwen2.5-Math-7B.
    \item \textbf{Code model:} Qwen2.5-7B-Instruct-1M (code generalization setting).
\end{itemize}
All models are trained with:
\begin{itemize}
    \item \textbf{Precision:} bfloat16 mixed precision.
    \item \textbf{Parallelization:} FSDP sharding (configured per model scale).
    \item \textbf{Memory optimization:} gradient checkpointing and padding removal.
\end{itemize}
The policy outputs token probabilities:
\begin{equation}
\pi_\theta(a_t \mid s_t) = \mathrm{softmax}\!\left(f_\theta(s_t)\right).
\end{equation}

\paragraph{AERO:}
AERO implements adaptive resource allocation through iterative rejection sampling:
\begin{enumerate}
    \item \textbf{Problem Categorization:} Each problem $p$ is classified based on current accuracy $\bar{c}_p$ with rebalance threshold $\tau = 0.5$:
\begin{align}
\text{Category}(p) =
\begin{cases}
\text{Zero} & \text{if } \bar{c}_p = 0, \\
\text{Partial} & \text{if } 0 < \bar{c}_p < \tau, \\
\text{Good} & \text{if } \tau \le \bar{c}_p < 1, \\
\text{Perfect} & \text{if } \bar{c}_p = 1.
\end{cases}
\end{align}
    
\item \textbf{Iterative Rejection Sampling:} 
For up to $K_{\text{max}} = 10$ iterations, we update the per-problem allocation as
\begin{align}
n_p^{(k+1)} =
\begin{cases}
n_p^{(k)} + \Delta n, & \text{if } \text{Category}(p) = \text{Zero}, \\[4pt]
\text{balance}\big(n_p^{(k)}\big), & \text{if } \text{Category}(p) = \text{Partial}, \\[4pt]
n_p^{(k)}, & \text{otherwise},
\end{cases}
\end{align}
where $\Delta n = 2$ is the rejection increment, and $\text{balance}(\cdot)$ adjusts the sample count to equalize correct and incorrect rollouts.

    \item \textbf{Budget Constraints:} Sample allocation respects both per-problem and total limits:
    \begin{align}
    n_p &\leq n_{\text{max}} = 32 \quad \forall p \\
    \sum_{p \in P} n_p &\leq N_{\text{total}} = n_{\text{budget}} \times |P|
    \end{align}
    where $n_{\text{budget}} = 16$ controls the total computational budget.
\end{enumerate}

\paragraph{Bayesian Posterior Updates (Robust Success-Rate Estimation).}
AERO maintains a Bayesian posterior over each prompt's success probability $u$ to obtain a robust estimate even in all-fail ($c=0$) or all-correct ($c=n$) cases.
We place a Beta prior on $u$,
\begin{equation}
u \sim \mathrm{Beta}(\alpha_0,\beta_0),
\end{equation}
and observe $c$ correct rollouts out of $n$ samples. By Beta--Binomial conjugacy, the posterior is
\begin{equation}
p(u \mid c,n) \sim \mathrm{Beta}(\alpha_0+c,\ \beta_0+n-c).
\end{equation}
We use the posterior mean as a smoothed success-rate estimate:
\begin{equation}
\tilde u \;=\; \mathbb{E}[u \mid c,n] \;=\; \frac{c+\alpha_0}{n+\alpha_0+\beta_0},
\end{equation}
where we set a weak prior $\alpha_0=\beta_0=1$ in practice.
This posterior smoothing ensures $\tilde u \in (0,1)$ even when $c\in\{0,n\}$. 


\paragraph{GRPO Training:}
The framework uses GRPO without critic networks:
\begin{align}
\mathcal{L}_{\text{GRPO}} &= \mathbb{E}_{(s,a) \sim \pi_{\theta_{\text{old}}}} \left[ \min\left( r_t(\theta) A_t, \text{clip}(r_t(\theta), 1-\epsilon, 1+\epsilon) A_t \right) \right]
\end{align}
where advantages $A_t$ are computed using group-relative baselines rather than learned value functions, and $r_t(\theta) = \frac{\pi_\theta(a_t|s_t)}{\pi_{\theta_{\text{old}}}(a_t|s_t)}$ is the importance ratio.
\paragraph{Hyperparameter Configuration.}
Table~\ref{tab:hyperparameters} summarizes the key hyperparameters used in our AERO implementation for both math and code domains. 
Unless specified, we use the same AERO allocation parameters ($n_{\text{total}}$, $n_{\text{explore}}$, $n_{\max}$) across domains.

\begin{table}[h]
\centering
\caption{AERO hyperparameters for math and code experiments.}
\label{tab:hyperparameters}
\setlength{\tabcolsep}{4pt}
\renewcommand{\arraystretch}{1.05}
\small
\begin{tabular}{lcc}
\toprule
\textbf{Parameter} & \textbf{Math} & \textbf{Code} \\
\midrule
\multicolumn{3}{c}{\textit{Model / System}} \\
Backbone & Qwen2.5-Math-(1.5B/7B) & Qwen2.5-7B-Instruct-1M \\
Precision & bfloat16 & bfloat16 \\
FSDP Sharding & (e.g., 2--4 way) & (e.g., 4 way) \\
Grad Checkpointing & True & True \\
\midrule
\multicolumn{3}{c}{\textit{Optimization}} \\
Learning rate & 5e-6 & 5e-6 \\
Warmup ratio & 0.03 & 0.03 \\
Grad clipping & 1.0 & 1.0 \\
Clip ratio $\epsilon$ & 0.2 & 0.2 \\
Entropy coef. & 1e-3 & 1e-3 \\
GRPO epochs & 1 & 1 \\
\midrule
\multicolumn{3}{c}{\textit{AERO Sampling}} \\
$n_{\text{explore}}$ & 8 & 8 \\
$n_{\text{total}}$ & 16 & 16 \\
$n_{\max}$ & 32 & 32 \\
$\Delta n$ & 2 & 4 \\
$K_{\max}$ & 10 & 3 \\
Threshold $\tau$ & 0.5 & 0.5 \\
\midrule
\multicolumn{3}{c}{\textit{Sequence / Data}} \\
Train batch size & 256 & 128  \\
Val batch size & 150 &  200 \\
Max prompt length & 1024 & 1024  \\
Max response length & 1536 & 1536  \\
Sampling temperature & 1.0 & 1.0 \\
\midrule
\multicolumn{3}{c}{\textit{Bayesian Smoothing}} \\
$\alpha_0,\beta_0$ & 1, 1 & 1, 1 \\
Advantage std & True & True \\
\bottomrule
\end{tabular}
\end{table}

\section{Data Collection and Prompt Engineering}
\label{sec:data}

AERO operates on structured mathematical reasoning datasets with systematic prompt engineering and automatic reward computation. 
Training data is collected on-policy during the GRPO training loop, with careful control over prompt design and reward evaluation.

\paragraph{System Prompt Design.}\mbox{}\\
\begin{greenbox}
\textbf{Math System Prompt:}  
"Your task is to solve the given question step by step. You should conduct a systematic, thorough reasoning process before providing the final answer. 
This involves analyzing, summarizing, exploring, reassessing, and refining your reasoning process through multiple iterations. 
Each reasoning step should include detailed analysis, verification, and refinement of ideas.  
Include the final answer in \texttt{\textbackslash boxed\{\}} for closed-form results such as multiple-choice or numerical answers."
\end{greenbox}
\vspace{5pt}\noindent
\begin{greenbox}
\textbf{Code System Prompt:}  
"When tackling complex reasoning tasks, you have access to the following actions.  
Use them as needed to progress through your thought process.

[ASSESS]  
[ADVANCE]  
[VERIFY]  
[SIMPLIFY]  
[SYNTHESIZE]  
[PIVOT]  
[OUTPUT]  

You should strictly follow the format below:  

[ACTION NAME]  
Your action step 1  
Your action step 2  
...  

Next action: [NEXT ACTION NAME]  

After reasoning, output the final executable Python code enclosed in a single code block, e.g.,  

\texttt{\string`\string`\string`python}\\
\textit{your solution code}\\
\texttt{\string`\string`\string`}"
\end{greenbox}

\paragraph{Example Problem.}
A representative training sample is shown below:

\begin{greenbox}
\textbf{User Query:} 
"A ship traveling along a river has covered 24 km upstream and 28 km downstream. 
For this journey, it took half an hour less than for traveling 30 km upstream and 21 km downstream... 
Determine the speed of the ship in still water and the speed of the river."

\textbf{Expected Response:} 
"Let $v$ be the speed of the ship in still water... [step-by-step reasoning] ... 
The speed of the ship in still water is $\boxed{10}$ km/h and the speed of the river is $\boxed{4}$ km/h."
\end{greenbox}

\paragraph{Example Problem.}\mbox{}\\[-0.3em]
A representative training sample (code) is shown below:

\begin{greenbox}
\textbf{User Query:}
"Given base $n$ and a string $k$ consisting of digits in $[0,n-1]$, interpret $k$ as the concatenation of base-$n$ digits written in decimal (digits $\ge 10$ are written as their decimal numerals without separators). Find the minimum $x \le 10^{18}$ whose base-$n$ representation produces exactly $k$."

\textbf{Expected Response:}
\begin{verbatim}
[ASSESS]
Parse k into tokens (each token is one base-n digit written in decimal,
no leading zeros unless token=="0").
Next action: [ADVANCE]

[ADVANCE]
DP over positions: dp[i]=min value for prefix k[0:i].
Transition by choosing next token k[i:j] as digit d<n:
dp[j]=min(dp[j], dp[i]*n + d), discard if >1e18.
Next action: [VERIFY]

[VERIFY]
Enforce token rules (no leading zeros), and cap transitions at 1e18.
Next action: [OUTPUT]

[OUTPUT]
```python
# ... final accepted solution ...
\end{verbatim}
\end{greenbox}

\paragraph{On-Policy Data Collection.}
During training, AERO collects experience tuples $(s_t, a_t, r_t, \log \pi_\theta(a_t|s_t))$:
\begin{itemize}
    \item $s_t$: Formatted prompt containing the system instruction and question
    \item $a_t$: Generated reasoning trace following the structured format
    \item $r_t$:$r_t \in \{0,1\}$, the Binary correctness reward
    \item $\log \pi_\theta(a_t|s_t)$: Log-probability of the generated trajectory
\end{itemize}

\section{Reward System}
\label{sec:reward}


\subsection{Multi-Domain Reward Framework}
AERO uses verifiable, domain-specific evaluators to provide reliable rewards across heterogeneous reasoning tasks. 
Given a model response $s$ and ground-truth answer $g$, the reward is computed by invoking the appropriate verifier for the task domain:
\begin{equation}
r_t = \mathcal{V}_d(s,g), \qquad d \in \{\text{math},\text{code},\text{science/QA}\},
\end{equation}
where $\mathcal{V}_{\text{math}}$ checks mathematical equivalence, $\mathcal{V}_{\text{code}}$ executes unit tests for program synthesis, and $\mathcal{V}_{\text{science/QA}}$ performs answer matching for structured QA (e.g., multiple-choice). 
Unless otherwise specified, we use a binary reward $r_t \in \{0,1\}$ indicating whether the response is verified as correct.

\subsection{Token-Level Reward Assignment}

Following standard RL practice, rewards are assigned sparsely to the final token of each response:
\begin{align}
\text{reward}[i, \text{len} - 1] = r_t.
\end{align}

\end{document}